\documentclass[11pt,letterpaper]{article}

\usepackage[letterpaper,margin=1in]{geometry}
\usepackage[T1]{fontenc}
\usepackage{times}
\usepackage{microtype}
\usepackage{amsmath,amssymb,amsthm,mathtools}
\usepackage{booktabs}
\usepackage{float}
\newcommand{\citep}{\cite}
\newcommand{\citet}{\cite}
\usepackage[ruled,vlined]{algorithm2e}
\usepackage[hidelinks]{hyperref}
\renewenvironment{abstract}
  {\begin{center}\bfseries\abstractname\end{center}\quotation\normalsize}
  {\endquotation}

\newtheorem{theorem}{Theorem}[section]
\newtheorem{proposition}[theorem]{Proposition}
\newtheorem{lemma}[theorem]{Lemma}
\newtheorem{corollary}[theorem]{Corollary}
\theoremstyle{remark}
\newtheorem{remark}[theorem]{Remark}
\newcommand{\inner}[2]{\left\langle #1,#2\right\rangle}
\newcommand{\R}{\mathbb{R}}
\newcommand{\E}{\mathbb{E}}
\newcommand{\cX}{\mathcal{X}}
\newcommand{\cY}{\mathcal{Y}}
\newcommand{\cF}{\mathcal{F}}
\newcommand{\AltReg}{\operatorname{Reg}^{\mathrm{alt}}}
\newcommand{\AltRegExp}{\operatorname{Reg}^{\mathrm{alt,exp}}}
\newcommand{\order}{\mathcal{O}}
\newcommand{\Var}{\operatorname{Var}}

\newcommand{\vol}{\operatorname{vol}}
\newcommand{\ip}[2]{\left\langle #1,#2\right\rangle}

\title{Minimax Alternating Regret \\for the Experts Problem and Online Convex Optimization}
\author{  Mengxiao Zhang \\
  University of Iowa \\  \texttt{mengxiao-zhang@uiowa.edu} \\}
\date{}

\begin{document}

\hypersetup{pageanchor=false}
\begin{titlepage}
\maketitle
\begin{abstract}
In this paper, we study alternating regret in online convex optimization (OCO), motivated by the success of alternating learning dynamics in two-player games. Although previous works have shown that $o(\sqrt{T})$ alternating regret is achievable under various assumptions on the loss functions and feasible domains, the minimax regret rate has remained open even for the expert problem. In this paper, we resolve this question by showing matching lower and upper bounds for both the expert problem and general OCO.

Somewhat surprisingly, for the $d$-expert problem, we show that the minimax alternating regret is $\Theta(\log d)$ by providing both a lower bound and an algorithm achieving a matching upper bound. This significantly improves upon the $\order(T^{1/3}\log^{2/3} d)$ established by~\citet{hait2025alternating}. We further extend our results to general OCO over a $d$-dimensional compact convex set and prove that the worst-case minimax alternating regret is $\Theta\left(d\log \left(1+\frac{T}{d}\right)\right)$, also significantly improving upon the $\order((d\log T)^{2/3}T^{1/3})$ upper bound and resolving the open problem posed by~\citet{cevher2023alternation,hait2025alternating}. As a direct implication, our alternating dynamics achieve $\order(\frac{\log d}{T})$ convergence to Nash equilibrium in two-player zero-sum games, matching the state of the art, and give the first such convergence rate to coarse correlated equilibrium in two-player general-sum games.

Technically, our upper bound for the expert problem is achieved by a corrected variant of Hedge. We extend the same corrected-potential argument to continuous action sets to obtain the optimal alternating-regret rate for OCO. For the lower bounds, the construction for the expert problem repeatedly eliminates half of the candidate experts, while the OCO lower bound instance construction replaces this discrete elimination by a more involved multiscale construction on the unit disk.

\end{abstract}

\thispagestyle{empty}
\end{titlepage}
\hypersetup{pageanchor=true}
\section{Introduction}\label{sec:introduction}

Alternating regret is a recently introduced performance measure for online convex optimization (OCO) \citep{cevher2023alternation,hait2025alternating}, motivated by the success of alternating learning dynamics in two-player games~\citep{tammelin2014solving,wibisono2022alternating}. Unlike classical regret, which evaluates each loss only at the action selected before the loss is revealed, alternating regret additionally evaluates the same loss at the learner's next action, selected after observing the loss. This criterion captures the dynamics of alternating play in two-player games: a player responds to the opponent's most recent move, while the same response also serves as its action against the opponent's subsequent update. As shown by~\citet{hait2025alternating}, when both players guarantee alternating regret at most $V_T$ over $T$ rounds, their average strategy profile forms an $\order(V_T/T)$-approximate equilibrium. This connection has motivated growing interest in understanding how small alternating regret can be in OCO.

\citet{cevher2023alternation} initiated the study of alternating regret in adversarial online linear optimization (OLO) and established $o(\sqrt{T})$ regret bounds over several important domains. Specifically, they proposed two algorithms that achieve $\order(T^{1/3}\log^{4/3}(dT))$ alternating regret for the $d$-expert problem and $\order(\log T)$ alternating regret over the unit $\ell_2$ ball, respectively. Subsequently, \citet{hait2025alternating} showed that Hedge achieves $\order(T^{1/3}\log^{2/3}d)$ alternating regret for the $d$-expert problem and extended the study to general OCO with arbitrary bounded convex losses. Together, these results demonstrate that alternating regret can have a substantially better dependence on the horizon than classical regret.

However, the bounds established in these works did not determine the minimax rates, even for the basic expert problem. This is in sharp contrast to classic regret, whose minimax rates are well understood. In particular, they left open whether a polynomial dependence on the horizon $T$ is unavoidable. This motivated the following question posed in~\citet{cevher2023alternation,hait2025alternating}:

\begin{center}
\emph{What is the minimax alternating regret for the expert problem and for general online convex optimization?}
\end{center}

\paragraph{Contributions.}
In this paper, we completely resolve this question by establishing matching lower and upper bounds in both settings. Somewhat surprisingly, for the $d$-expert problem, we show that the minimax alternating regret is $\Theta(\log d)$, independent of the horizon $T$. For general OCO with bounded convex losses over a $d$-dimensional compact convex set, we prove that the worst-case minimax alternating regret is $\Theta(d\log(1+\frac{T}{d}))$. Our main results are summarized informally below.

\begin{theorem}[Informal]
For every $d\geq 2$ and every horizon $T\geq \log_2 d$, the minimax alternating regret for the $d$-expert problem is $\Theta(\log d)$. For OCO with bounded convex losses, the worst-case minimax alternating regret over $d$-dimensional compact convex sets is $\Theta(d\log(1+\frac{T}{d}))$.
\end{theorem}
More specifically, our contributions are as follows:

\begin{itemize}
\item In Section~\ref{sec:experts}, we study the $d$-expert problem and first establish an $\Omega(\log d)$ lower bound using a halving construction in which the adversary reveals the identity of the best expert one bit at a time. Although this logarithmic lower bound may initially appear weak, we show that it is in fact tight by introducing a corrected variant of Hedge with matching $\order(\log d)$ alternating regret. Importantly, this guarantee is independent of $T$, revealing a sharp distinction from classical regret and resolving the open problem posed by~\citet{cevher2023alternation,hait2025alternating}. By applying the alternating-regret-to-equilibrium reduction of~\citet{hait2025alternating}, our result yields $\order(\frac{\log d}{T})$ convergence to Nash equilibrium in two-player zero-sum games, matching the state of the art~\citep{rakhlin2013predictable,syrgkanis2015fast}, and the first such convergence rate to coarse correlated equilibrium in two-player general-sum games. Technically, our algorithm is based on the standard exponential-weights algorithm but importantly modifies the played distribution using carefully designed linear and quadratic correction terms. These corrections cancel the local norm terms that arise in the analysis and allow the resulting potential argument to telescope without accumulating a dependence on $T$.

\item In Section~\ref{sec:oco}, we extend our results to general OCO. We first prove an $\Omega\left(d\log \left(1+\frac{T}{d}\right)\right)$ lower bound using a multiscale geometric construction on the unit two-dimensional disk, organized over a sequence of epochs with geometrically decreasing spatial scales. In each epoch, the construction creates a constant expected loss gap while ensuring that a single final comparator remains nearly as competitive on all previous epochs. A direct-product argument then extends this two-dimensional lower bound to dimension $d$. To provide a matching upper bound, we extend our algorithm for the $d$-expert problem to a continuous action space, and we show that the same corrected-potential argument yields a matching alternating-regret bound, resolving the open problem proposed in~\citet{hait2025alternating}.

\end{itemize}

\subsection{Related Works}

\paragraph{Online convex optimization.}
Prediction with expert advice and online convex optimization are classical models for sequential decision making. Weighted Majority, Hedge, and exponential weights are foundational algorithms for the expert problem \citep{littlestone1994weighted,freund1997decision,cesa2006prediction}, while OCO extends the action space from a finite set of experts to a convex domain \citep{zinkevich2003online,hazan2016introduction}. The classical objective in this literature is standard regret, and a broad line of work develops sharper guarantees by exploiting additional structure such as curvature or predictability of the loss sequence \citep{hazan2010variation,rakhlin2013predictable}. Alternating regret studies a different performance measure: the loss \(f_t\) is evaluated not only at the action \(x_t\) chosen before \(f_t\) is revealed, but also at the next action \(x_{t+1}\), which may depend on \(f_t\). Consequently, alternating regret can be substantially smaller than standard regret even when the loss sequence is fully adversarial.

\paragraph{Alternating regret and alternating learning dynamics.}
Alternating updates have long been used in game solving, most notably in CFR+ and related methods for extensive-form games \citep{tammelin2014solving,burch2019revisiting}. Their theoretical advantages have also been studied in normal-form and convex-concave zero-sum games \citep{bailey2020finite,wibisono2022alternating,katona2026symplectic}. \citet{wibisono2022alternating} connected the acceleration produced by alternation to a regret quantity that was subsequently formalized as alternating regret by \citet{cevher2023alternation}. The latter work studied alternating regret against an arbitrary adversarial sequence in online linear optimization and established \(o(\sqrt{T})\) guarantees over several important domains. \citet{hait2025alternating} extended this adversarial perspective to general OCO with bounded convex losses and showed how alternating-regret guarantees translate directly into convergence guarantees for alternating learning dynamics in two-player zero-sum and general-sum convex-concave games. In contrast to analyses that rely on the particular coupled dynamics of a game, these results isolate a worst-case online-learning property of each player. Our work completes this line by identifying the minimax alternating-regret rates for both the expert problem and general bounded-loss OCO.

\paragraph{Fast convergence in games.}
The connection between regret minimization and equilibrium computation is classical: no-regret learning yields convergence of time-averaged play to equilibrium, and the usual \(\order(\sqrt{T})\) regret guarantee leads to the canonical \(\order(1/\sqrt{T})\) convergence rate \citep{freund1999adaptive}. A large literature obtains faster rates by exploiting the structure of game-generated loss sequences, with optimistic and predictive methods providing a particularly successful route \citep{rakhlin2013predictable,syrgkanis2015fast,chen2020hedging, daskalakis2021nearoptimal,farina2022nearoptimal}. For two-player zero-sum games, such methods can achieve total regret independent of $T$ and hence an $\order(1/T)$ average-iterate convergence rate. Alternation provides a distinct route to acceleration: under alternating play, equilibrium error is controlled by the players' average alternating regret \citep{wibisono2022alternating,hait2025alternating}. Recent work further demonstrates that optimism and alternation are genuinely different mechanisms: \citet{lazarsfeld2025optimism} show that, without regularization, optimistic Fictitious Play attains constant regret in \(2\times2\) zero-sum games with a unique interior Nash equilibrium, whereas Alternating Fictitious Play can suffer \(\Omega(\sqrt{T})\) regret even on Matching Pennies.

Subsequent to this work, a line of work has shown constant-regret guarantees under simultaneous full-information self-play, which yield $\order(1/T)$ convergence to coarse correlated equilibrium in general-sum games with any fixed number of players and actions~\citep{liu2026constant,abbadi2026constant,soleymani2026multiplicative,ha2026horizon}, and constant swap-regret guarantees, which further yield $\order(1/T)$ convergence to correlated equilibrium~\citep{mai2026constant}. Our results show that $\order(\frac{\log d}{T})$ convergence to coarse correlated equilibrium in two-player general-sum games can be obtained via alternation.

\paragraph{Concurrent work.} In concurrent work, \citet{tao2026optimal} establish the same minimax alternating-regret rates for experts and bounded-loss OCO, together with the corresponding two-player equilibrium guarantees. Their alternation-aware Hedge update places a centered quadratic loss term inside the exponential weights and determines its center through a scalar optimization. Our update instead applies explicit linear and quadratic corrections to the standard exponential-weights distribution, derived from a variance correction to the log-partition potential. The expert lower bounds also differ: their construction requires a stronger adaptive adversary and uses $T\geq d-1$, whereas our halving construction is oblivious, and applies for $T\geq\log_2 d$. Both OCO lower bounds use multiscale constructions on a two-dimensional disk followed by a product construction.

\section{Preliminaries}\label{sec:prelim}

\paragraph{Notation.}

For an integer $m\ge1$, denote by $[m]$ the set $\{1,\ldots,m\}$, let $\Delta_m=\{p\in\R^m_+:\sum_{i=1}^m p_i=1\}$ be the probability simplex, and let $e_i$ denote its $i$-th standard basis vector. With a slight abuse of notation, given a $d$-dimensional compact convex set $\Omega\subseteq\R^d$, we also define $\Delta_{\Omega}\triangleq\{p:\Omega\to\R_+: \int_{\Omega}p(x)\,dx=1\}$ to be the set of probability densities over $\Omega$, and denote by $\vol(\Omega)$ its Lebesgue volume. We use $\mathbf{0}$ for the all-zero vector in an appropriate dimension (or the zero function, when the context is clear). Inner products are denoted by $\ip{x}{y}$, and $\|\cdot\|_2$ denotes the Euclidean norm. We denote by $\partial A$ the boundary of a set $A$.

For a distribution $q\in\Delta_m$, a subset $S\subseteq[m]$, and a vector $v\in\R^m$, we use
\[
  q(S)\triangleq\sum_{i\in S}q_i,
  \qquad
  \E_q[v]=\E_{i\sim q}[v_i]\triangleq\sum_{i=1}^m q_i v_i,
  \qquad
  \Var_q(v)=\Var_{i\sim q}(v_i)
  \triangleq\E_{i\sim q}\!\left[
    \bigl(v_i-\E_q[v]\bigr)^2
  \right].
\]
The same notation is used for expectations and variances under a probability density $q\in\Delta_\Omega$, with sums replaced by integrals. For a score vector $R\in\R^m$, let $q[R]\in\Delta_m$ denote the exponential-weights distribution with $q_i[R]\propto\exp(R_i)$. For a score function $R:\Omega\to\R$, we use the same notation $q[R]\in\Delta_\Omega$ for the density $q[R](x)\propto\exp(R(x))$. For a differentiable convex function $F:\R^m\to\R$, its Bregman divergence is $ D_F(a,b) \triangleq F(a)-F(b)-\ip{\nabla F(b)}{a-b}. $ We use $\widetilde{\order}(\cdot)$ to suppress logarithmic factors.

\paragraph{Problem setting.}

Let $\cX$ be a $d$-dimensional compact convex set. The interaction between the learner and the environment lasts for $T>2$ rounds. At each round $t\in[T]$, the learner chooses $x_t\in\cX$, while simultaneously the adversary picks a convex function $f_t:\cX\to[-1,1]$ and reveals it to the learner. Let $(\cF_t)_{t=0}^T$ denote the learner's filtration, where $\cF_0$ is the $\sigma$-algebra generated by the initial information and the learner's internal random seed, and, for each $t\in[T]$, $\cF_t \triangleq \sigma\bigl(\cF_0,f_1,\ldots,f_t\bigr)$ is the $\sigma$-algebra generated by $\cF_0$ and the revealed loss functions $f_1,\ldots,f_t$. Thus the action $x_t$ must be $\cF_{t-1}$-measurable and the adversary must commit to $f_t$ independently of the learner's fresh round-$t$ sampling randomness. After observing $f_T$, the learner additionally chooses $x_{T+1}$ based on $\cF_T$.

In this paper, we consider alternating regret as the performance measure. Specifically, the alternating regret against a comparator $u\in\cX$ is defined as $\AltReg_T(u) \triangleq \sum_{t=1}^T \bigl(f_t(x_t)+f_t(x_{t+1})-2f_t(u)\bigr)$, and we define $\AltReg_T\triangleq \sup_{u\in\cX}\AltReg_T(u)$. We adopt the endpoint convention $f_0(x)=f_{T+1}(x)=0$ for all $x\in\cX$ and define the effective loss $ h_t=f_{t-1}+f_t$ for $t=1,\ldots,T+1$. Each $f_t$ appears once in $h_t(x_t)$, once in $h_{t+1}(x_{t+1})$, and twice at the comparator. Consequently, we can equivalently write $\AltReg_T(u)$ as follows:
\begin{equation}\label{eq:endpoint-identity}
  \AltReg_T(u)
  =\sum_{t=1}^{T+1}\bigl(h_t(x_t)-h_t(u)\bigr).
\end{equation}
We only assume that the convex losses satisfy $f_t\in[-1,1]$ and do not assume Lipschitzness, smoothness, or strong convexity.

\paragraph{Expert Problem.} In this work, we also consider the expert problem, which is the linear specialization of the OCO protocol. In this case, $\cX=\Delta_d$, and the loss function $f_t$ is induced by a loss vector $\ell_t\in[-1,1]^d$ defined as $ f_t(p)=\ip{p}{\ell_t}. $ Following the same convention, we set $\ell_0=\ell_{T+1}=\mathbf{0}$ and define the effective loss vector $ z_t=\ell_{t-1}+\ell_t$ for all $t=1,\ldots,T+1$. Thus $h_t(p)=\ip{p}{z_t}$, and the expert alternating regret with respect to comparator $u\in\Delta_d$ is defined as
\begin{equation}\label{eq:expert-regret}
  \AltRegExp_T(u)\triangleq\sum_{t=1}^{T+1}\ip{p_t-u}{z_t},
\end{equation}
and we also define $\AltRegExp_T \triangleq \max_{u\in\Delta_d}\AltRegExp_T(u)$.

\section{Alternating Regret in the Expert Problem}
\label{sec:experts}

We begin with the $d$-expert problem and derive tight lower and upper bounds.

\subsection{Lower Bound}
\label{sec:expert-lower}

In this section, we show that the alternating-regret lower bound for the $d$-expert problem is $\Omega(\log d)$ even when the adversary is oblivious, as formally stated in the following theorem.

\begin{theorem}\label{thm:expert-lower}
For every \(T\ge\log_2 d\), \(d\ge2\), and every possibly randomized expert algorithm, there is an oblivious sequence \(\ell_1,\ldots,\ell_T\in[-1,1]^d\) for which
$$
  \mathbb{E}\bigl[\AltRegExp_T\bigr]
  \ge
  \lfloor\log_2 d\rfloor.
$$
Here, the expectation is taken with respect to the internal randomization of the learner.
\end{theorem}

Although this lower bound may seem weak because it does not depend on $T$, we will show in Section~\ref{sec:expert-upper} that this rate is in fact tight. The hard loss sequence construction has a simple interpretation. Before the game begins, the adversary chooses a nested sequence of sets, each of which is a uniformly random half of the preceding set. Each active round reveals one more bit about the final comparator. Conditional on the losses revealed in previous rounds, the learner cannot determine which half will survive. Consequently, the learner assigns, in expectation, only half of its probability mass on the current set to the surviving half. After \(\log_2 d\) cuts only one expert remains, and there are no more bits to hide.

\begin{proof}
Let \(m=\lfloor\log_2 d\rfloor\), and choose an initial set \(S_0\subseteq[d]\) of \(2^m\) experts. To prove the existence of the desired oblivious sequence, first consider the following randomized oblivious construction. Before the game begins, sample a uniformly random permutation \(\pi\) of the elements of \(S_0\), and define
$$
  S_t
  =
  \bigl\{\pi(1),\ldots,\pi(2^{m-t})\bigr\},
  \qquad t\in[m].
$$
Thus, conditional on \(S_{t-1}\), the set \(S_t\) is a uniformly random subset of \(S_{t-1}\) of cardinality \(\frac12|S_{t-1}|\). On every active round \(t\le m\), define \(\ell_t\in[-1,1]^d\) by
\begin{equation}\label{eq:halving-loss}
  \ell_{t,i}
  =
  \begin{cases}
    -1,&i\in S_t,\\
    +1,&i\notin S_t.
  \end{cases}
\end{equation}
All later losses are zero, meaning that \(\ell_t=\mathbf{0}\) for \(t\in\{m+1,\ldots,T\}\). Since the permutation \(\pi\), and hence the entire loss sequence, is sampled before the learner chooses \(p_1\), this adversary is oblivious.

Recall that $\mathcal{F}_{t-1}$ contains the history immediately after \(p_t\) has been chosen and before \(\ell_t\) is revealed. The past losses reveal \(S_1,\ldots,S_{t-1}\), but they do not reveal the ordering of the elements remaining in \(S_{t-1}\). Moreover, \(p_t\) depends only on the past losses and the learner's internal randomization. Therefore, conditional on \(\mathcal{F}_{t-1}\), the set \(S_t\) remains a uniformly random half of \(S_{t-1}\). It follows that $\mathbb{E}\bigl[p_t(S_t)\mid\mathcal{F}_{t-1}\bigr]=\frac12p_t(S_{t-1}). $

Finally, let \(i^\star=\pi(1)\) denote the unique element of \(S_m\). Because \(i^\star\in S_t\) for every \(t\in[m]\), we have \(\ell_{t,i^\star}=-1\) on every active round. Consequently, $ \ip{p_t}{\ell_t}-\ell_{t,i^\star} = 2\bigl(1-p_t(S_t)\bigr). $ Taking the conditional expectation and using \(p_t(S_{t-1})\le1\), we obtain that
\begin{align*}
  \mathbb{E}\left[
    \ip{p_t}{\ell_t}-\ell_{t,i^\star}
    \,\middle|\,\mathcal{F}_{t-1}
  \right]=
  2\left(
    1-\frac12p_t(S_{t-1})
  \right) =
  2-p_t(S_{t-1})\ge 1.
\end{align*}
The next distribution \(p_{t+1}\) is chosen after \(\ell_t\) is known, but its loss cannot be smaller than the minimum coordinate of \(\ell_t\). Therefore,
$$
  \ip{p_{t+1}}{\ell_t}-\ell_{t,i^\star}
  =
  2\bigl(1-p_{t+1}(S_t)\bigr)
  \ge0.
$$
Thus, each active round contributes at least one unit of alternating regret in expectation. Summing over the \(m\) active rounds gives $ \mathbb{E}_{\pi,\mathrm{alg}} \bigl[\AltRegExp_T\bigr] \ge m. $ Since this inequality holds after averaging over the finitely many permutations of \(S_0\), there exists a fixed permutation \(\bar{\pi}\) such that the corresponding deterministic loss sequence satisfies $ \mathbb{E}_{\mathrm{alg}} \bigl[\AltRegExp_T\bigr] \ge m. $ This loss sequence is fixed before the interaction begins and is therefore oblivious, which proves the claim.
\end{proof}

\subsection{Upper Bound}
\label{sec:expert-upper}

We now turn to the upper bound for the expert problem. Somewhat surprisingly, the $\Omega(\log d)$ lower bound from Section~\ref{sec:expert-lower} is achievable without any dependence on the horizon, revealing a sharp separation from classical regret. Our algorithm is based on exponential weights, but modifies the distribution played on each round by linear and quadratic correction terms. We first give a walkthrough of the algorithm and then explain the design for these corrections from the one-step condition needed for logarithmic alternating regret.

\paragraph{Algorithm walkthrough.}
Our algorithm maintains a score vector $R_t$ and the corresponding exponential-weights distribution $q_t$ at each round $t$. At the beginning of round $t$, the learner knows $q_{t-1}$ and the previously revealed loss $\ell_{t-1}$. Starting from $q_{t-1}$, the learner adjusts the weights using each expert's previous loss relative to the average. Denote by $\bar\ell_{t-1}\triangleq\E_{i\sim q_{t-1}}[\ell_{t-1,i}]$ the mean loss at round $t-1$, $\widetilde\ell_{t-1,i}\triangleq\ell_{t-1,i}-\bar\ell_{t-1}$ the centered loss, and $H_{t-1}\triangleq\E_{i\sim q_{t-1}}[\widetilde\ell_{t-1,i}^2]$ its variance under $q_{t-1}$. The learner then plays
\begin{align}
\label{eq:played-distribution}
p_{t,i}
=
q_{t-1,i}\cdot
\left(
1-4\eta\widetilde\ell_{t-1,i}
-2\eta^2\left(\widetilde\ell_{t-1,i}^2-H_{t-1}\right)
\right),
\end{align}
where $\eta>0$ is the learning rate. Specifically, \eqref{eq:played-distribution} incorporates two correction terms. The linear correction increases the weights of experts with below-average losses in the previous round and decreases those with above-average losses. The quadratic correction further decreases weights when the squared deviation exceeds $H_{t-1}$ and increases them when it is smaller. Both corrections have zero mean under $q_{t-1}$, so the resulting weights sum to one. With the choice of $\eta=\frac{1}{10}$, we show in Proposition~\ref{prop:expert-feasibility} in Appendix~\ref{app:expert-potential} that the weights are also nonnegative, making $p_t$ a valid distribution. After observing $\ell_t$, the algorithm updates $R_t=R_{t-1}-(\ell_{t-1}+\ell_t)$, and lets $q_t$ be the exponential-weights distribution induced by $R_t$. The terminal distribution $p_{T+1}$ is computed in the same way from the known state $(q_T,\ell_T)$ after $\ell_T$ is revealed. The complete pseudocode is shown in Algorithm~\ref{alg:althedge}.

\begin{algorithm}[H]
\caption{Corrected AltHedge}
\label{alg:althedge}
\DontPrintSemicolon
\KwIn{Number of experts $d$ and learning rate $\eta>0$}
Set $\ell_0=R_0=\mathbf{0}\in\R^d$, and let $q_0$ be uniform on $[d]$.\;
\For{$t=1,\ldots,T$}{
  Compute
  {  $$
    \bar\ell_{t-1}=\E_{i\sim q_{t-1}}[\ell_{t-1,i}],
    \qquad
    \widetilde\ell_{t-1,i}=\ell_{t-1,i}-\bar\ell_{t-1},
    \qquad
    H_{t-1}=\E_{i\sim q_{t-1}}[\widetilde\ell_{t-1,i}^2].
  $$

  }
  Compute $p_t\in\Delta_d$ according to \eqref{eq:played-distribution}.\;

  Play $p_t$ and observe $\ell_t\in[-1,1]^d$.\;

  Set $R_t=R_{t-1}-(\ell_t+\ell_{t-1})$ and $q_{t}=q[\eta R_{t}]$.
}
\end{algorithm}

At first sight, the two correction terms in \eqref{eq:played-distribution} may appear specialized. We next explain in the remainder of this section the reason for the choice of the two corrections by comparing this update with an ideal distribution that has access to the current loss. For the analysis, we use the endpoint convention $\ell_{T+1}=\mathbf{0}$ and set $R_{T+1}=R_T-\ell_T$. Recall that $z_t=\ell_{t-1}+\ell_t$, so that $R_t=-\sum_{\tau=1}^t z_\tau$ for all $t\in[T+1]$.

\paragraph{The ideal post-update iterate.}
To understand the design of $p_t$, we first consider the distribution that would be used if the current effective loss $z_t$ were already known, as in the standard Be-the-Leader argument. Define the log-partition function and its gradient by
\begin{equation}
\label{eq:log-partition}
F(R)
\triangleq
\frac{1}{\eta}
\log\left(\frac{1}{d}\sum_{i=1}^d e^{\eta R_i}\right),
\qquad
q[\eta R]\triangleq\nabla F(R).
\end{equation}
By construction, $q_t=q[\eta R_t]$. Since $R_t=R_{t-1}-z_t$, convexity of $F$ gives $\langle q_t,z_t\rangle\le F(R_{t-1})-F(R_t)$. Therefore, we have for every expert $i\in[d]$,
\begin{equation}
\label{eq:ideal-qt-bound}
\sum_{t=1}^{T+1}\langle q_t-e_i,z_t\rangle
\le
-F(R_{T+1})+R_{T+1,i}
\le
\frac{\log d}{\eta},
\end{equation}
where the last inequality follows from $F(R)\ge\max_iR_i-(\log d)/\eta$. Thus, if the learner could play the post-update distribution $q_t$, the desired logarithmic bound would follow immediately.

However, the caveat is that $q_t$ is not $\cF_{t-1}$ measurable and depends on $z_t=\ell_{t-1}+\ell_t$, and hence on the unrevealed loss $\ell_t$. At the beginning of round $t$, the latest available exponential-weights iterate is $q_{t-1}$. If the learner plays this iterate, then $\langle q_{t-1},z_t\rangle = F(R_{t-1})-F(R_t)+D_F(R_t,R_{t-1})$, and consequently
\begin{equation}
\label{eq:preupdate-hedge-bound}
\sum_{t=1}^{T+1}\langle q_{t-1}-e_i,z_t\rangle
\le
\frac{\log d}{\eta}
+
\sum_{t=1}^{T+1}D_F(R_t,R_{t-1}).
\end{equation}
The additional Bregman-divergence sum is the standard Hedge stability cost. In general, this term is of order $\eta\sum_t\Var_{q_{t-1}}(z_t)$ and may contribute $\order(\eta T)$. Therefore, the natural $\cF_{t-1}$-measurable choice $q_{t-1}$ recovers only the classical $\order(\frac{\log d}{\eta}+\eta T)$ guarantee.

Unlike in the standard expert problem, however, the learner is not completely ignorant of the effective loss $z_t$ when selecting $p_t$: its component $\ell_{t-1}$ is already $\cF_{t-1}$-measurable. We will then see how this property can be utilized to design a specific \emph{potential function} so that the $\cF_{t-1}$-measurable distribution $p_t$ enjoys the same telescoping bound as the ideal post-update iterate.

\paragraph{A corrected-potential criterion.}
The preceding derivation suggests the following sufficient condition. It is enough to find a correction sequence $C_t$ satisfying $C_0=C_{T+1}=0$ and
\begin{align}
\langle p_t,z_t\rangle
\le
F(R_{t-1})-F(R_t)+C_t-C_{t-1}
\label{eq:sufficient-credit-condition}
\end{align}
on every round $t$. Indeed, summing \eqref{eq:sufficient-credit-condition} makes both the log-partition terms and the corrections telescope. Using $F(R_0)=0$, we obtain $\sum_{t=1}^{T+1}\langle p_t,z_t\rangle\le-F(R_{T+1})$. Combining this with $F(R_{T+1})\ge\max_iR_{T+1,i}-(\log d)/\eta$ gives the desired regret bound. Thus, \eqref{eq:sufficient-credit-condition} is the one-step inequality that we aim to establish.

More precisely, for each round $t$, we aim to construct a concave function $G:[0,1]\to\mathbb{R}$ satisfying $G(0)=F(R_{t-1})-C_{t-1}$, $G(1)=F(R_t)-C_t$, and $G'(0)=-\langle p_t,z_t\rangle$, where $p_t$ depends only on information available before $\ell_t$ is revealed. Concavity then gives $G(1)-G(0)\le G'(0)=-\langle p_t,z_t\rangle$, which rearranges to \eqref{eq:sufficient-credit-condition}. The use of the left-endpoint derivative is important because its coefficients can depend only on the pre-update state and are thus $\cF_{t-1}$ measurable, whereas the right endpoint depends on the unrevealed loss $\ell_t$.

To construct such a function, we consider the following family of corrected potentials, which will determine both the correction and the played distribution. For $\alpha>0$, define
\begin{align}
U_\alpha(R,\ell)
\triangleq
F(R)-\alpha\eta\Var_{q[\eta R]}(\ell),
\label{eq:parametric-corrected-potential}
\end{align}
where $R\in\R^d$ represents a score vector and $\ell\in\R^d$ represents a single-round loss vector.

\paragraph{Intuition for the potential function.}
We next explain why the potential in \eqref{eq:parametric-corrected-potential} takes this particular form. Recall that the log-partition term $F(R)$ has positive curvature, which gives rise to the positive Bregman-divergence term in \eqref{eq:preupdate-hedge-bound}. Concretely, for an arbitrary effective-loss direction $z$, define
\begin{align}
R(\theta)
\triangleq
R_{t-1}-\theta z,
\qquad
\theta\in[0,1].
\label{eq:alternating-path}
\end{align}
Along this score path, a direct calculation shows that
\begin{align*}
\frac{\mathrm{d}^2}{\mathrm{d}\theta^2}F(R(\theta))
=
\eta\Var_{q[\eta R(\theta)]}(z)
\ge0.
\end{align*}
The reason for introducing the negative variance term in \eqref{eq:parametric-corrected-potential} is then to offset this positive curvature. To see why this correction is natural, define $\ell(\theta)\triangleq\ell_{t-1}-\theta\cdot z$. For the realized direction $z=z_t=\ell_{t-1}+\ell_t$, this path joins the known loss $\ell_{t-1}$ to $-\ell_t$ and stays in $[-1,1]^d$. Since variance is unchanged by negation, $U_\alpha(R_t,-\ell_t)=U_\alpha(R_t,\ell_t)$, so these endpoints give the potential difference needed in \eqref{eq:sufficient-credit-condition}. For now, suppose that $q[\eta R(\theta)]$ were held fixed along the path. Then we have that
\begin{align*}
\frac{\mathrm{d}^2}{\mathrm{d}\theta^2}
\Var_q(\ell(\theta))
=
2\Var_q(z).
\end{align*}
Hence the correction $-\alpha\eta\Var_q(\ell(\theta))$ contributes a negative curvature $-2\alpha\eta\Var_q(z)$. For $G_\alpha(\theta)\triangleq U_\alpha(R(\theta),\ell(\theta))$, this fixed-$q$ calculation would therefore give the curvature to be $(1-2\alpha)\eta\Var_q(z)$. Thus, any $\alpha>1/2$ would reverse the positive curvature of $F$.

Moreover, the derivative of the corrected potential at the known left endpoint also naturally induces a candidate distribution. By the chain rule, we have that the derivative of $G_\alpha$ at $\theta=0$ is equal to
\begin{align*}
G_\alpha'(0)
=
-\left\langle
\nabla_RU_\alpha(R_{t-1},\ell_{t-1})
+\nabla_\ell U_\alpha(R_{t-1},\ell_{t-1}),
z
\right\rangle.
\end{align*}
Accordingly, define the coefficient vector
\begin{align*}
\widetilde p_t^{(\alpha)}
\triangleq
\nabla_RU_\alpha(R_{t-1},\ell_{t-1})
+\nabla_\ell U_\alpha(R_{t-1},\ell_{t-1}).
\end{align*}
Since this vector depends only on $(R_{t-1},\ell_{t-1})$, it is $\cF_{t-1}$-measurable. It also automatically has unit mass. Indeed, for every $c\in\R$, the potential satisfies $U_\alpha(R+c\mathbf{1},\ell)=U_\alpha(R,\ell)+c$ and $U_\alpha(R,\ell+c\mathbf{1})=U_\alpha(R,\ell)$. Differentiating these identities with respect to $c$ gives $\langle\nabla_RU_\alpha(R,\ell),\mathbf{1}\rangle=1$ and $\langle\nabla_\ell U_\alpha(R,\ell),\mathbf{1}\rangle=0$, and hence $\sum_i\widetilde p_{t,i}^{(\alpha)}=1$. These are the two desired properties for our choice of $p_t$ and it remains only to choose $\alpha$ and $\eta$ so that $\widetilde p_t^{(\alpha)}$ is coordinatewise nonnegative.

\paragraph{Concrete choices of $\alpha$ and $\eta$.}
Because $q[\eta R(\theta)]$ also varies along the path, the previous fixed-$q[\eta R(\theta)]$ curvature calculation above is only heuristic. However, with the help of the boundedness of the loss interpolation $\ell(\theta)$ and a more involved calculation, we can indeed show in Proposition~\ref{prop:parametric-corrected-potential} in Appendix~\ref{app:expert-potential} that the complete second-order derivative of $U_\alpha(R(\theta),\ell(\theta))$ is upper bounded as follows:
\begin{align}
\frac{\mathrm{d}^2}{\mathrm{d}\theta^2}
U_\alpha(R(\theta),\ell(\theta))
\le
\left[
(1-2\alpha)\eta
+8\alpha\eta^2
+3\alpha\eta^3
\right]
\Var_{q[\eta R(\theta)]}(z).
\label{eq:parametric-curvature-bound}
\end{align}
Therefore, choosing $\alpha=2$ and $\eta=1/10$ makes the coefficient on the right-hand side strictly negative. We then fix these two values and denote $U_2$ by $U$. For notational convenience, however, we use $\eta$ in the formulas below to represent the fixed value $1/10$, so that the dependence of the algorithm and the potential on the learning rate remains explicit. Thus, we define
\begin{align}
U(R,\ell)
\triangleq
F(R)-2\eta\Var_{q[\eta R]}(\ell).
\label{eq:expert-potential}
\end{align}
Set $C_t\triangleq2\eta\Var_{q_t}(\ell_t)$. Then we have that $U(R_t,\ell_t)=F(R_t)-C_t$. We also have $C_0=C_{T+1}=0$ as $\ell_0=\ell_{T+1}=\mathbf{0}$. Moreover, \eqref{eq:parametric-curvature-bound} shows that $G_2(\theta)$ is concave on $[0,1]$.

Now we compute $G_2'(0)$ and verify that this indeed induces a valid distribution. Based on this choice of potential, direct differentiation gives that
\begin{align*}
G_2'(0)
=
-\sum_{i=1}^d
q_{t-1,i}
\left[
1
-4\eta\widetilde\ell_{t-1,i}
-2\eta^2
\left(
\widetilde\ell_{t-1,i}^2-H_{t-1}
\right)
\right]z_i.
\end{align*}
The coefficient vector in $-G_2'(0)$ is exactly the distribution $p_t$ in \eqref{eq:played-distribution}. As shown above, its coordinates sum to one, while Proposition~\ref{prop:expert-feasibility} establishes their nonnegativity at $\eta=1/10$. Hence $p_t\in\Delta_d$.

The following lemma summarizes the two properties needed in the corrected-potential argument.

\begin{lemma}
\label{lem:correction-potential-concavity}
Fix a round $t$, a score vector $R_{t-1}\in\mathbb{R}^d$, and a loss vector $\ell_{t-1}\in[-1,1]^d$. For any $z\in\mathbb{R}^d$ satisfying $\ell_{t-1}-z\in[-1,1]^d$, let $R(\theta)$ be as in \eqref{eq:alternating-path} and let $\ell(\theta)$ be the loss interpolation introduced above. If $\eta=1/10$, then:
\begin{enumerate}
\item $G_2$ is concave on $[0,1]$;
\item $G_2'(0)=-\langle p_t,z\rangle$, where $p_t$ is the $\cF_{t-1}$-measurable distribution in \eqref{eq:played-distribution}.
\end{enumerate}
\end{lemma}

The complete proof of Lemma~\ref{lem:correction-potential-concavity} is deferred to Appendix~\ref{app:expert-potential}. Applying Lemma~\ref{lem:correction-potential-concavity} to the realized direction $z=z_t$ and using $U(R_t,-\ell_t)=U(R_t,\ell_t)$, we know that
\begin{align}
\langle p_t,z_t\rangle
\le
U(R_{t-1},\ell_{t-1})-U(R_t,\ell_t) =
F(R_{t-1})-F(R_t)+C_t-C_{t-1},
\label{eq:corrected-potential-inequality}
\end{align}
which is exactly the target inequality \eqref{eq:sufficient-credit-condition} that we want.

\paragraph{Regret guarantee.}
Having derived the correction and the played distribution, we now state the resulting alternating-regret guarantee.

\begin{theorem}
\label{thm:expert-upper}
For every adaptive sequence $\ell_1,\ldots,\ell_T\in[-1,1]^d$, Algorithm~\ref{alg:althedge} with $\eta=1/10$ satisfies $\AltRegExp_T\le10\log d$.
\end{theorem}

\begin{proof}
Summing \eqref{eq:corrected-potential-inequality} over $t=1,\ldots,T+1$ yields $\sum_{t=1}^{T+1}\langle p_t,z_t\rangle \le U(R_0,\ell_0)-U(R_{T+1},\ell_{T+1})$. Since $R_0=\ell_0=\mathbf{0}$, we have $U(R_0,\ell_0)=0$. Moreover, $\ell_{T+1}=\mathbf{0}$, so $U(R_{T+1},\ell_{T+1})=F(R_{T+1})$. Therefore, $\sum_{t=1}^{T+1}\langle p_t,z_t\rangle\le-F(R_{T+1})$. Using the effective-loss representation of alternating regret,
\begin{align*}
\AltRegExp_T
=
\sum_{t=1}^{T+1}\langle p_t,z_t\rangle
+\max_{i\in[d]}R_{T+1,i} \le
-F(R_{T+1})+\max_{i\in[d]}R_{T+1,i}
\le
\frac{\log d}{\eta}.
\end{align*}
Substituting $\eta=1/10$ completes the proof.
\end{proof}

To our knowledge, Theorem~\ref{thm:expert-upper} gives the first $\order(\log d)$ alternating-regret guarantee, which is independent of the horizon $T$. Together with the $\Omega(\log d)$ lower bound in Theorem~\ref{thm:expert-lower}, it establishes the minimax alternating regret as $\Theta(\log d)$ and improves the previous $\order(T^{1/3}\log^{2/3} d)$ upper bound of~\citet{hait2025alternating}. By applying the reduction of \citet{hait2025alternating}, the logarithmic alternating-regret bound yields an $\order(\frac{\log d}{T})$ convergence rate for alternating learning in two-player games. Specifically, in two-player zero-sum games, our result yields $\order(\frac{\log d}{T})$ convergence to Nash equilibrium. This matches the average-iterate convergence rate achieved by optimistic multiplicative-weight update under simultaneous play~\citep{rakhlin2013predictable,syrgkanis2015fast}. Moreover, in two-player general-sum games, it gives the first $\order(\frac{\log d}{T})$ convergence rate to coarse correlated equilibrium. The following corollary follows directly from the reduction of~\citet{hait2025alternating}.

\begin{corollary}
\label{cor:expert-games}
Consider a two-player normal-form game in which the players have $d_x$ and $d_y$ actions, respectively, and losses in $[-1,1]$. Suppose both players use Algorithm~\ref{alg:althedge} with $\eta=1/10$ in the alternating learning dynamic of \citet{hait2025alternating}. Then the empirical distribution assigning mass $1/(2T)$ to each of the $2T$ indexed profiles $\{(x_t,y_t),(x_{t+1},y_t)\}_{t\in[T]}$ is an $\varepsilon_{\mathrm{CCE}}$-approximate coarse correlated equilibrium with
\begin{align*}
  \varepsilon_{\mathrm{CCE}}
  =
  \order\left(
    \frac{
      \log(d_xd_y)
    }{T}
  \right).
\end{align*}
If the game is zero-sum, the time-averaged strategies form an $\varepsilon_{\mathrm{NE}}$-approximate Nash equilibrium with
\begin{align*}
  \varepsilon_{\mathrm{NE}}
  =
  \order\left(
    \frac{
      \log(d_xd_y)
    }{T}
  \right).
\end{align*}
\end{corollary}

\section{Alternating Regret for Online Convex Optimization}
\label{sec:oco}

Now we shift our focus to the more general OCO problem. In Section~\ref{sec:oco-lower}, we first show the lower-bound construction. Compared with the expert case, the main additional difficulty is that convexity allows the learner to hedge continuously between candidate comparators and a more careful geometric construction is required to force a loss gap despite this interpolation. We then give a matching upper bound in Section~\ref{sec:oco-upper} by extending the finite exponential-weights distribution to the continuous setting.

\subsection{Lower Bound}
\label{sec:oco-lower}

Our lower-bound construction can be viewed as a geometric analogue of the expert lower bound. In the expert problem, at each stage we keep a uniformly random half of the surviving experts, with all choices made before play. A linear loss can assign a small loss to every expert in the surviving subset and a large loss to every expert in the other subset, so the learner incurs a constant loss in expectation relative to any expert that remains. Repeating this elimination over logarithmically many stages gives the logarithmic lower bound. However, this argument does not extend directly to a fixed-dimensional convex domain. If the candidate comparators are represented by points in a convex set, an arbitrary assignment of small losses to one subset and large losses to another may not be realizable by a convex function.

We overcome this issue on the two-dimensional unit disk by providing a different hard instance construction. Fix a boundary point \(w\) on the two-dimensional unit disk, and let \(u\) be a small counterclockwise rotation of \(w\) satisfying $ 1-\langle w,u\rangle=\delta>0. $ We use two convex losses $\mathsf H^{(1)}_w$ and $\mathsf H^{(2)}_{w,\delta}$. The first one $\mathsf H^{(1)}_w$ mildly favors staying at \(w\), while the second one $\mathsf H^{(2)}_{w,\delta}$ strongly favors moving from \(w\) to \(u\). Their values at the two relevant points are
$$
\begin{array}{c|cc}
  & w & u\\ \hline
  \mathsf H^{(1)}_w & 0 & \delta/2\\
  \mathsf H^{(2)}_{w,\delta} & 1/2 & 0 .
\end{array}
$$
The second loss is sampled with probability $ p=\delta/(1+\delta), $ and the first loss is sampled otherwise. This choice ensures that, before the loss is revealed, every point in the disk has expected loss at least $p/2$. Thus the learner cannot avoid the uncertainty by choosing an intermediate point between \(w\) and \(u\).

This gives a two-point hard instance at a single spatial scale $\delta$. To amplify the lower bound, we repeat the construction over a sequence of epochs with geometrically decreasing scales $\delta_1,\delta_2,\ldots,\delta_{J}$. In epoch $j\in[J]$, we use the corresponding losses $\mathsf H^{(1)}_{w_j}$ and $\mathsf H^{(2)}_{w_j,\delta_j}$ for $ n_j=\Theta(1/\delta_j) $ rounds. If no $\mathsf H^{(2)}_{w_j,\delta_j}$ loss appears, the reference point remains at $w_j$; otherwise we move the reference point to a nearby boundary point $w_{j+1}$. This creates a constant expected loss gap in each epoch. However, these gaps cannot be summed directly because regret is measured against a single fixed comparator. We therefore decrease the scales geometrically, so that all later movements are much smaller than the movement at the current epoch. As a result, each epoch still contributes a positive expected loss gap against the final-epoch reference point. Hence the constant gaps from different epochs can be accumulated against one fixed comparator. Since the epoch lengths $n_j$ grow geometrically, only $\Theta(\log T)$ epochs fit within $T$ rounds, yielding an $\Omega(\log T)$ lower bound in two dimensions. Finally, taking a product of $\Theta (d)$ such disks and assigning a separate time segment to each coordinate pair yields $ \Omega\left(d\log(1+\frac{T}{d})\right). $

Formally, we have the following theorem and we defer the complete proof to Appendix~\ref{app:oco-lower}.

\begin{theorem}\label{thm:oco-lower}
There exists a universal constant $c>0$ such that, for every $d\ge 2$ and every horizon $T>2$, there exists a $d$-dimensional compact convex set $\cX\subseteq\R^d$ with the following property: for every possibly randomized learner, there exists an oblivious sequence of continuous convex loss functions $f_1,\dots,f_T:\cX\to[0,1]$ such that
\begin{equation}\label{eq:oco-lower-main}
\E[\AltReg_T]\ge c\cdot d\log\left(1+\frac{T}{d}\right),
\end{equation}
where the expectation is over the learner's internal randomness.
\end{theorem}

\subsection{Upper Bound: a Continuous Correction Algorithm}
\label{sec:oco-upper}

We now turn to the algorithm design for obtaining a matching upper bound. Our proposed algorithm, Algorithm~\ref{alg:continuous-althedge}, is essentially the continuous counterpart of Algorithm~\ref{alg:althedge}. The score vector $R_t$ is replaced by a score function $R_t:\cX\to\R$, the exponential-weights distribution is replaced by the density $q_t=q[\eta R_t]$, and finite expectations are replaced by integrals over $\cX$. Apart from these changes, the two correction terms in \eqref{eq:continuous-played-density} are identical to those in Algorithm~\ref{alg:althedge}.

\begin{algorithm}[h]
\caption{Continuous Corrected AltHedge}
\label{alg:continuous-althedge}
\DontPrintSemicolon
\KwIn{a $d$-dimensional compact convex set $\cX$ and learning rate $\eta>0$}
Set $f_0=R_0=0$ and $q_0(x)=1/\vol(\cX)$.\;
\For{$t=1,\ldots,T$}{
  Compute
  \begin{align*}
    \bar f_{t-1}
    &\triangleq
    \int_{\cX}f_{t-1}(x)q_{t-1}(x)\,dx,\\
    \widetilde f_{t-1}(x)
    &\triangleq
    f_{t-1}(x)-\bar f_{t-1},\\
    H_{t-1}
    &\triangleq
    \int_{\cX}\widetilde f_{t-1}(x)^2q_{t-1}(x)\,dx.
  \end{align*}

  Compute $p_t\in\Delta_{\cX}$ according to
  \begin{align}
    p_t(x)
    \triangleq
    q_{t-1}(x)
    \left[
      1
      -4\eta \widetilde f_{t-1}(x)
      -2\eta^2\bigl(\widetilde f_{t-1}(x)^2-H_{t-1}\bigr)
    \right].
    \label{eq:continuous-played-density}
  \end{align}

  Play
  $x_t=
  \begin{cases}
    \displaystyle \int_{\cX}x p_t(x)\,dx, & \text{\rm (Option I)},\\[2mm]
    X_t,\quad X_t\sim p_t, & \text{\rm (Option II)},
  \end{cases}$
  and observe $f_t:\cX\to[-1,1]$.\;

  Set $R_t=R_{t-1}-(f_t+f_{t-1})$, and set $q_t=q[\eta R_t]$.
}
\end{algorithm}

As in the expert case, both correction terms in \eqref{eq:continuous-played-density} integrate to zero under $q_{t-1}$, and $p_t$ is a valid probability density when $\eta=\frac{1}{10}$, as formally verified in Proposition~\ref{prop:continuous-feasibility}. After constructing the corrected density $p_t$, there are two natural ways to turn it into an action. Option I plays its mean, $x_t=\int_{\cX}x p_t(x)dx\in\cX$, while Option II samples $x_t\sim p_t$. The first choice gives a deterministic strategy and a deterministic regret guarantee, but requires computing the mean of $p_t$, which is generally an integration problem. The second avoids this computation and only requires sampling from $p_t$, but its regret guarantee holds only in expectation.

In particular, for Option II, sampling is particularly convenient. Since $R_t$ is concave, $q_t$ is log-concave. Given the exact values of $\bar f_{t-1}$ and $H_{t-1}$, a sample from $p_t$ can be obtained by rejection sampling from $q_{t-1}$ with constant acceptance probability. Thus Option II incurs only a constant rejection-sampling overhead beyond sampling from the underlying continuous exponential-weights density. A more concrete discussion is deferred to Appendix~\ref{app:oco-upper-details}. By contrast, implementing Option I requires additionally estimating the mean of $p_t$, for example using repeated samples. Our regret analysis below considers the exact versions of both options.

For the analysis, we use the same endpoint convention as in the expert case: set $f_{T+1}=0$ and $R_{T+1}=R_T-f_T$. We compute $\bar f_T,\widetilde f_T,H_T$, and $p_{T+1}$ from $(q_T,f_T)$ using the same formulas as above. Under Option I, set $x_{T+1}=\int_{\cX}x p_{T+1}(x)\,dx$; under Option II, sample $x_{T+1}\sim p_{T+1}$. The guarantee of Algorithm~\ref{alg:continuous-althedge} is formally stated below. Both options achieve the minimax rate from Theorem~\ref{thm:oco-lower}.

\begin{theorem}
\label{thm:oco-upper}
For every $d$-dimensional compact convex set $\cX$ and every adaptive sequence of convex losses $f_1,\ldots,f_T:\cX\to[-1,1]$, Algorithm~\ref{alg:continuous-althedge} with $\eta=1/10$ satisfies
\begin{align*}
  \AltReg_T
  &\le
  \order\left(d\log\left(1+\frac{T}{d}\right)\right)
  &&\text{under Option I},\\
  \E[\AltReg_T]
  &\le
  \order\left(d\log\left(1+\frac{T}{d}\right)\right)
  &&\text{under Option II},
\end{align*}
where the expectation in the second bound is over the learner's internal randomness.
\end{theorem}

Theorem~\ref{thm:oco-upper} significantly improves the $\widetilde{\order}(d^{2/3}T^{1/3})$ alternating-regret guarantee of \citet{hait2025alternating} for general bounded convex losses. Together with Theorem~\ref{thm:oco-lower}, it characterizes the minimax alternating regret. Similar to the expert case, the connection between alternating regret and alternating learning dynamics established by \citet{hait2025alternating} immediately yields improved convergence rates in convex games. In particular, using Option I gives deterministic convergence guarantees.

\begin{corollary}
\label{cor:oco-games}
Consider a two-player game with action sets $\cX$ and $\cY$ of dimensions $d_x$ and $d_y$, respectively, and losses in $[-1,1]$ that are convex in each player's own action. Suppose both players use Option I of Algorithm~\ref{alg:continuous-althedge} with $\eta=1/10$ in the alternating learning dynamic of \citet{hait2025alternating}. Then the empirical distribution assigning mass $1/(2T)$ to each of the $2T$ indexed profiles $\{(x_t,y_t),(x_{t+1},y_t)\}_{t\in[T]}$ is an $\varepsilon_{\mathrm{CCE}}$-approximate coarse correlated equilibrium with
\begin{align*}
  \varepsilon_{\mathrm{CCE}}
  =
  \order\left(
    \frac{
        d_x\log(1+\frac{T}{d_x})+
        d_y\log(1+\frac{T}{d_y})
    }{T}
  \right).
\end{align*}
If the game is convex-concave and zero-sum, the time-averaged strategies form an $\varepsilon_{\mathrm{NE}}$-approximate Nash equilibrium with
\begin{align*}
  \varepsilon_{\mathrm{NE}}
  =
  \order\left(
    \frac{
      d_x\log(1+\frac{T}{d_x})
      +
      d_y\log(1+\frac{T}{d_y})
    }{T}
  \right).
\end{align*}
\end{corollary}

\subsection{Proof of Theorem~\ref{thm:oco-upper}}
\label{sec:oco-upper-analysis}

The proof is almost the same as the expert analysis in Section~\ref{sec:expert-upper}. Recall that the expert proof uses the log-partition potential \eqref{eq:log-partition} and the corrected potential \eqref{eq:expert-potential}. Their continuous counterparts are obtained by replacing the uniform average over the $d$ experts by the normalized Lebesgue integral over $\cX$. Specifically, for a score function $R:\cX\to\R$, with a slight abuse of notation, we define
\begin{align}
  F(R)
  &\triangleq
  \frac{1}{\eta}
  \log\left(
    \frac{1}{\vol(\cX)}
    \int_{\cX}\exp(\eta R(x))\,dx
  \right),
  \label{eq:continuous-potential}\\
  U(R,f)
  &\triangleq
  F(R)-2\eta\Var_{x\sim q[\eta R]}(f(x)).
  \label{eq:continuous-corrected-potential}
\end{align}
Thus $F$ is the continuous analogue of \eqref{eq:log-partition}, while $U$ is the continuous analogue of \eqref{eq:expert-potential}.

The first ingredient is the direct continuous counterpart of the one-step corrected-potential inequality \eqref{eq:corrected-potential-inequality}. It shows that the same potential decrease controls the loss averaged under the corrected density.

\begin{lemma}
\label{lem:continuous-corrected-potential}
Define $h_t\triangleq f_{t-1}+f_t$ for $t\in[T+1]$. Then, for every $t\in[T+1]$,
\begin{align}
  \int_{\cX}p_t(x)h_t(x)\,dx
  \le
  U(R_{t-1},f_{t-1})-U(R_t,f_t).
  \label{eq:continuous-corrected-potential-inequality}
\end{align}
\end{lemma}

The proof uses the same interpolation as Lemma~\ref{lem:correction-potential-concavity}. The only additional analytic step is to justify differentiation under the integral. Other than this, the derivative calculation is identical to the expert case. The complete proof is given in Appendix~\ref{app:oco-upper-details}.

After telescoping \eqref{eq:corrected-potential-inequality}, the expert proof compares the terminal log-partition potential with the best expert through $\max_iR_{T+1,i}\le F(R_{T+1})+(\log d)/\eta$. An individual point $u\in\cX$ has zero volume, so this discrete comparison has no literal continuous analogue. The following lemma replaces it by comparing $u$ with a small scaled copy of $\cX$ around it.

\begin{lemma}
\label{lem:continuous-comparator-bound}
For every $u\in\cX$ and every $\delta\in(0,1)$,
\begin{align}
  F(R_{T+1})
  \ge
  R_{T+1}(u)-4T\delta-\frac{d}{\eta}\log\frac{1}{\delta}.
  \label{eq:continuous-comparator-bound}
\end{align}
\end{lemma}

This is the only additional ingredient beyond the expert proof. The proof considers $(1-\delta)u+\delta\cX$, whose volume is $\delta^d\vol(\cX)$, and uses convexity and boundedness to compare the cumulative loss throughout this set with the loss at $u$. The complete proof is deferred to Appendix~\ref{app:oco-upper-details}.

We are now ready to prove Theorem~\ref{thm:oco-upper}.

\begin{proof}[Proof of Theorem~\ref{thm:oco-upper}]
We prove the result for Option II. The guarantee for Option I follows from the same argument together with Jensen's inequality.

Exactly as in the proof of Theorem~\ref{thm:expert-upper}, summing \eqref{eq:continuous-corrected-potential-inequality} over $t=1,\ldots,T+1$ and using $U(R_0,f_0)=0$ and $f_{T+1}=0$ gives $\sum_{t=1}^{T+1}\int_{\cX}p_t(x)h_t(x)\,dx\le-F(R_{T+1})$. Under Option II, $p_t$ is determined by $f_1,\dots,f_{t-1}$, and $x_t\sim p_t$. Therefore, we have $\E[h_t(x_t)\mid f_1,\dots,f_t]=\int_{\cX}p_t(x)h_t(x)\,dx$. Hence, by the tower property,
\begin{align*}
  \E\left[\sum_{t=1}^{T+1}h_t(x_t)\right]
  \le
  -\E[F(R_{T+1})].
\end{align*}

For each realized loss sequence, fix any $u\in\cX$. By Lemma~\ref{lem:continuous-comparator-bound}, for every $\delta\in(0,1)$, $-F(R_{T+1})\le-R_{T+1}(u)+4T\delta+(d/\eta)\log(1/\delta)$. Moreover, by the definition of the terminal score, $R_{T+1}(u)=-2\sum_{t=1}^T f_t(u)$. Taking the infimum over $u\in\cX$ and combining these inequalities with the endpoint identity $\sum_{t=1}^{T+1}h_t(x_t)=\sum_{t=1}^T(f_t(x_t)+f_t(x_{t+1}))$ gives
\begin{align}
  \E[\AltReg_T]
  \le
  4T\delta+\frac{d}{\eta}\log\frac{1}{\delta}.
  \label{eq:continuous-delta-bound}
\end{align}

When $T\ge2d$, choose $\delta=d/T$ and recall that $\eta=1/10$. Substituting these choices into \eqref{eq:continuous-delta-bound} gives $\E[\AltReg_T]\le4d+10d\log(T/d)=\order(d\log(1+T/d))$. When $T<2d$, boundedness of the losses gives the deterministic bound $\AltReg_T\le4T=\order(d\log(1+T/d))$. This proves the claimed bound for Option II. Finally, under Option I, convexity of $h_t$ and Jensen's inequality give $h_t(x_t)\le\int_{\cX}p_t(x)h_t(x)\,dx$ at every round. Repeating the same argument without taking expectations therefore gives $\AltReg_T\le\order(d\log(1+T/d))$ under Option I.
\end{proof}

\section{Conclusion}

In this paper, we study the minimax alternating regret in the expert problem and online convex optimization. For the $d$-expert problem, somewhat surprisingly, we achieve $\order(\log d)$ alternating regret, independent of the horizon $T$, with a matching lower bound. For $d$-dimensional OCO with bounded convex losses, we establish the minimax rate $\Theta\left(d\log\left(1+\frac{T}{d}\right)\right)$. These results significantly improve the previous $T^{1/3}$-dependent upper bounds and resolve the minimax rates in both settings. Our upper bounds are based on a corrected-potential argument for Hedge and its continuous analogue, while the lower bounds use a halving construction for experts and a multiscale geometric construction for OCO. As immediate consequences, the improved alternating-regret guarantees also yield faster convergence rates for alternating learning dynamics in games. In particular, they give the first $\order(\frac{\log d}{T})$ convergence rate to coarse correlated equilibrium in two-player general-sum games.  Interesting directions for future work include obtaining similarly sharp alternating-regret guarantees under partial-information feedback and extending the approach to games with more than two players.

\newpage

\appendix
\section{Omitted Details in Section~\ref{sec:experts}}
\label{app:expert-potential}

This appendix verifies the feasibility of the corrected distribution and provides the derivative calculation used in Lemma~\ref{lem:correction-potential-concavity}. We first verify that the correction in \eqref{eq:played-distribution} indeed defines a probability distribution.

\begin{proposition}
\label{prop:expert-feasibility}
For every $t=1,\ldots,T+1$, the vector $p_t$ defined in \eqref{eq:played-distribution} with $\eta=1/10$ belongs to $\Delta_d$.
\end{proposition}

\begin{proof}
Both correction terms in \eqref{eq:played-distribution} are centered under $q_{t-1}$, so $\sum_i p_{t,i}=1$. Moreover, $|\widetilde\ell_{t-1,i}|\le2$ and $H_{t-1}\ge0$. Hence, for $\eta=1/10$, the multiplicative correction factor is at least $1-8\eta-8\eta^2>0$. Therefore every coordinate of $p_t$ is nonnegative, and $p_t\in\Delta_d$.
\end{proof}

We next establish the curvature calculation used to choose the variance correction. The proposition below gives both derivatives of the parametric corrected potential and the upper bound on its curvature used in Section~\ref{sec:expert-upper}.

\begin{proposition}
\label{prop:parametric-corrected-potential}
Fix a round $t$, a score vector $R_{t-1}\in\R^d$, a loss vector $\ell_{t-1}\in[-1,1]^d$, and $z\in\R^d$ such that $\ell_{t-1}-z\in[-1,1]^d$. For $\theta\in[0,1]$, let $R(\theta)\triangleq R_{t-1}-\theta z$ and $\ell(\theta)\triangleq \ell_{t-1}-\theta z$. For $\alpha>0$, define $U_\alpha(R,\ell)\triangleq F(R)-\alpha\eta\Var_{q[\eta R]}(\ell)$ and $G_\alpha(\theta)\triangleq U_\alpha(R(\theta),\ell(\theta))$.

Let $\E_\theta[\cdot]$ denote expectation under $q[\eta R(\theta)]$, and denote the centered variables by $\widetilde z(\theta)\triangleq z-\E_\theta[z]$ and $\widetilde\ell(\theta)\triangleq \ell(\theta)-\E_\theta[\ell(\theta)]$. Then
\begin{align}
G_\alpha'(\theta)
=-\E_\theta[z]
+2\alpha\eta\E_\theta[\widetilde z(\theta)\widetilde\ell(\theta)]
+\alpha\eta^2\E_\theta[\widetilde z(\theta)\widetilde\ell(\theta)^2],
\label{eq:G-alpha-first-derivative}
\end{align}
and
\begin{align}
G_\alpha''(\theta)
&=
(1-2\alpha)\eta\E_\theta[\widetilde z(\theta)^2]
-4\alpha\eta^2\E_\theta[\widetilde z(\theta)^2\widetilde\ell(\theta)]
\notag\\
&\qquad-\alpha\eta^3\E_\theta[\widetilde z(\theta)^2\widetilde\ell(\theta)^2]
+\alpha\eta^3\E_\theta[\widetilde z(\theta)^2]
  \E_\theta[\widetilde\ell(\theta)^2]
\notag\\
&\qquad+2\alpha\eta^3\E_\theta[\widetilde z(\theta)\widetilde\ell(\theta)]^2.
\end{align}
In particular, we have
\begin{align}\label{eq:G-alpha-curvature-bound}
G_\alpha''(\theta)
\le
\left[(1-2\alpha)\eta+8\alpha\eta^2+3\alpha\eta^3\right]
\Var_{q[\eta R(\theta)]}(z).
\end{align}
\end{proposition}

\begin{proof}
Fix a round $t$. Throughout the proof, we keep the round index only in $R_{t-1}$ and $\ell_{t-1}$, and omit it from the quantities defined along the interpolation. Denote the distribution and its expectation operator by $q(\theta)\triangleq q[\eta R(\theta)]$ and $\E_\theta[\cdot]\triangleq\E_{i\sim q(\theta)}[\cdot]$.

We first record a differentiation rule that will be used repeatedly. Define the scalar $\bar z(\theta)\triangleq\E_\theta[z]$ and the centered vector $\widetilde z(\theta)\in\R^d$ by $\widetilde z_i(\theta)\triangleq z_i-\bar z(\theta)$. Since $R_i'(\theta)=-z_i$, differentiating the exponential-weights distribution gives $q_i'(\theta)=-\eta q_i(\theta)\widetilde z_i(\theta)$. Consequently, for any differentiable indexed family $\phi(\theta)=(\phi_i(\theta))_{i=1}^d$,
\begin{align}
\frac{\mathrm d}{\mathrm d\theta}\E_\theta[\phi(\theta)]
=
\sum_iq_i'(\theta)\phi_i(\theta)
+\sum_iq_i(\theta)\phi_i'(\theta)=
\E_\theta[\phi'(\theta)]
-\eta\E_\theta[\widetilde z(\theta)\phi(\theta)].
\label{eq:moving-expectation}
\end{align}
Products and powers inside $\E_\theta[\cdot]$ are interpreted coordinatewise. For example, $\E_\theta[\widetilde z(\theta)^2] =\sum_iq_i(\theta)\widetilde z_i(\theta)^2$ and $\E_\theta[\widetilde z(\theta)\phi(\theta)] =\sum_iq_i(\theta)\widetilde z_i(\theta)\phi_i(\theta)$ are scalars.

We next derive the centered quantities needed below. Define the scalar $\bar\ell(\theta)\triangleq\E_\theta[\ell(\theta)]$ and the centered vector $\widetilde\ell(\theta)\in\R^d$ by $\widetilde\ell_i(\theta)\triangleq \ell_i(\theta)-\bar\ell(\theta)$. Thus $z=\widetilde z(\theta)+\bar z(\theta)\mathbf 1$ and $\ell(\theta)=\widetilde\ell(\theta)+\bar\ell(\theta)\mathbf 1$.

Applying \eqref{eq:moving-expectation} with $\phi(\theta)=z$, and noting that $z$ does not depend on $\theta$, gives $\bar z'(\theta) =-\eta\E_\theta[\widetilde z(\theta)z]$. Since $\E_\theta[\widetilde z(\theta)]=0$, we obtain $\bar z'(\theta) =-\eta\E_\theta[\widetilde z(\theta)^2]$. Therefore, for every coordinate $i$, $\widetilde z_i'(\theta) =\eta\E_\theta[\widetilde z(\theta)^2]$.

Next apply \eqref{eq:moving-expectation} with $\phi(\theta)=\ell(\theta)$. Since $\ell'(\theta)=-z$, we have $\bar\ell'(\theta) =-\bar z(\theta) -\eta\E_\theta[\widetilde z(\theta)\ell(\theta)]$. Using $\ell(\theta)=\widetilde\ell(\theta)+\bar\ell(\theta)\mathbf 1$ and $\E_\theta[\widetilde z(\theta)]=0$ gives $\E_\theta[\widetilde z(\theta)\ell(\theta)] =\E_\theta[\widetilde z(\theta)\widetilde\ell(\theta)]$. Hence $\bar\ell'(\theta) =-\bar z(\theta) -\eta\E_\theta[\widetilde z(\theta)\widetilde\ell(\theta)]$. Differentiating $\widetilde\ell_i(\theta)=\ell_i(\theta)-\bar\ell(\theta)$ then yields
\[
\widetilde\ell_i'(\theta)
=
-\widetilde z_i(\theta)
+\eta\E_\theta[
  \widetilde z(\theta)\widetilde\ell(\theta)
].
\]

The derivatives of the log-partition term follow immediately. Since $\frac{\mathrm d}{\mathrm d\theta}F(R(\theta))=-\bar z(\theta)$, the identity for $\bar z'(\theta)$ gives $ \frac{\mathrm d^2}{\mathrm d\theta^2}F(R(\theta)) = \eta\E_\theta[\widetilde z(\theta)^2] = \eta\Var_{i\sim q(\theta)}(z_i). $

We now differentiate the variance term. Let $H(\theta)\triangleq \Var_{q[\eta R(\theta)]}(\ell(\theta)) =\E_\theta[\widetilde\ell(\theta)^2]$. Applying \eqref{eq:moving-expectation} and the expression for $\widetilde\ell_i'(\theta)$ gives
\[
H'(\theta)
=
-2\E_\theta[
  \widetilde z(\theta)\widetilde\ell(\theta)
]
-\eta\E_\theta[
  \widetilde z(\theta)\widetilde\ell(\theta)^2
].
\]

We can now obtain the first derivative of the corrected potential. Since $G_\alpha(\theta)=F(R(\theta))-\alpha\eta H(\theta)$, combining the expressions for $\frac{\mathrm d}{\mathrm d\theta}F(R(\theta))$ and $H'(\theta)$ gives
\begin{align*}
G_\alpha'(\theta)
=
-\bar z(\theta)
+2\alpha\eta
\E_\theta[
  \widetilde z(\theta)\widetilde\ell(\theta)
]
+\alpha\eta^2
\E_\theta[
  \widetilde z(\theta)\widetilde\ell(\theta)^2
].
\end{align*}

To differentiate $H'(\theta)$ once more, we apply \eqref{eq:moving-expectation} to the two expectations above. For the first term, using the expressions for $\widetilde z_i'(\theta)$ and $\widetilde\ell_i'(\theta)$ derived above gives
\begin{align*}
\frac{\mathrm d}{\mathrm d\theta}
\E_\theta[
  \widetilde z(\theta)\widetilde\ell(\theta)
]
=
\E_\theta\left[
  \frac{\mathrm d}{\mathrm d\theta}
  \bigl(
    \widetilde z(\theta)\widetilde\ell(\theta)
  \bigr)
\right]
-\eta\E_\theta[
  \widetilde z(\theta)^2\widetilde\ell(\theta)
]=
-\E_\theta[\widetilde z(\theta)^2]
-\eta\E_\theta[
  \widetilde z(\theta)^2\widetilde\ell(\theta)
].
\end{align*}
Here the remaining terms vanish because $\E_\theta[\widetilde z(\theta)] =\E_\theta[\widetilde\ell(\theta)]=0$.

Similarly, applying \eqref{eq:moving-expectation} to $\phi(\theta)=\widetilde z(\theta)\widetilde\ell(\theta)^2$ gives
\begin{align*}
\frac{\mathrm d}{\mathrm d\theta}
\E_\theta[
  \widetilde z(\theta)\widetilde\ell(\theta)^2
]
&=
\E_\theta\left[
  \widetilde z'(\theta)\widetilde\ell(\theta)^2
  +
  2\widetilde z(\theta)\widetilde\ell(\theta)
  \widetilde\ell'(\theta)
\right]
-\eta\E_\theta[
  \widetilde z(\theta)^2\widetilde\ell(\theta)^2
]\\
&=
\eta\E_\theta[\widetilde z(\theta)^2]
     \E_\theta[\widetilde\ell(\theta)^2]
-2\E_\theta[
  \widetilde z(\theta)^2\widetilde\ell(\theta)
]+2\eta
\E_\theta[
  \widetilde z(\theta)\widetilde\ell(\theta)
]^2
-\eta\E_\theta[
  \widetilde z(\theta)^2\widetilde\ell(\theta)^2
].
\end{align*}

Differentiating $H'(\theta)$ and substituting the two identities above first gives
\begin{align*}
H''(\theta)
&=
-2
\frac{\mathrm d}{\mathrm d\theta}
\E_\theta[
  \widetilde z(\theta)\widetilde\ell(\theta)
]
-
\eta
\frac{\mathrm d}{\mathrm d\theta}
\E_\theta[
  \widetilde z(\theta)\widetilde\ell(\theta)^2
]
\\
&=
-2\left(
-\E_\theta[\widetilde z(\theta)^2]
-\eta\E_\theta[
  \widetilde z(\theta)^2\widetilde\ell(\theta)
]
\right) \\
&\qquad
-
\eta\left(
\eta\E_\theta[\widetilde z(\theta)^2]
     \E_\theta[\widetilde\ell(\theta)^2]
-2\E_\theta[
  \widetilde z(\theta)^2\widetilde\ell(\theta)
]+2\eta\E_\theta[
  \widetilde z(\theta)\widetilde\ell(\theta)
]^2
-\eta\E_\theta[
  \widetilde z(\theta)^2\widetilde\ell(\theta)^2
]
\right).
\end{align*}

Collecting terms, we obtain that
\begin{align*}
H''(\theta)
&=
2\E_\theta[\widetilde z(\theta)^2]
+4\eta\E_\theta[
  \widetilde z(\theta)^2\widetilde\ell(\theta)
]
\\
&\quad +
\eta^2\Bigl(
\E_\theta[
  \widetilde z(\theta)^2\widetilde\ell(\theta)^2
]
-\E_\theta[\widetilde z(\theta)^2]
 \E_\theta[\widetilde\ell(\theta)^2]-2\E_\theta[
  \widetilde z(\theta)\widetilde\ell(\theta)
]^2
\Bigr).
\end{align*}

Recall that $ \frac{\mathrm d^2}{\mathrm d\theta^2}F(R(\theta)) = \eta\E_\theta[\widetilde z(\theta)^2]. $ Since $ G_\alpha(\theta) = F(R(\theta))-\alpha\eta H(\theta), $ we therefore have
\begin{align*}
G_\alpha''(\theta)
&=
\frac{\mathrm d^2}{\mathrm d\theta^2}F(R(\theta))
-\alpha\eta H''(\theta)\\
&=
\eta\E_\theta[\widetilde z(\theta)^2]
-\alpha\eta
\Bigl(
2\E_\theta[\widetilde z(\theta)^2]
+4\eta\E_\theta[
  \widetilde z(\theta)^2\widetilde\ell(\theta)
]\\
&\qquad
+\eta^2\bigl(
\E_\theta[
  \widetilde z(\theta)^2\widetilde\ell(\theta)^2
]
-\E_\theta[\widetilde z(\theta)^2]
 \E_\theta[\widetilde\ell(\theta)^2]
-2\E_\theta[
  \widetilde z(\theta)\widetilde\ell(\theta)
]^2
\bigr)
\Bigr).
\end{align*}
Collecting terms gives
\begin{align}
G_\alpha''(\theta)
&=
(1-2\alpha)\eta
\E_\theta[\widetilde z(\theta)^2]
-
4\alpha\eta^2
\E_\theta[
  \widetilde z(\theta)^2\widetilde\ell(\theta)
]
\nonumber\\
&\qquad
-
\alpha\eta^3
\E_\theta[
  \widetilde z(\theta)^2\widetilde\ell(\theta)^2
]
+
\alpha\eta^3
\E_\theta[\widetilde z(\theta)^2]
\E_\theta[\widetilde\ell(\theta)^2]
\nonumber\\
&\qquad
+
2\alpha\eta^3
\E_\theta[
  \widetilde z(\theta)\widetilde\ell(\theta)
]^2.
\label{eq:G-alpha-second-derivative}
\end{align}

It remains to upper bound the curvature. Since $\ell(\theta)\in[-1,1]^d$, we have $|\widetilde\ell_i(\theta)|\le2$ and $\E_\theta[\widetilde\ell(\theta)^2] =\Var_{i\sim q(\theta)}(\ell_i(\theta))\le1$. By Cauchy--Schwarz, $\E_\theta[ \widetilde z(\theta)\widetilde\ell(\theta) ]^2 \le \E_\theta[\widetilde z(\theta)^2]$, and $|\E_\theta[ \widetilde z(\theta)^2\widetilde\ell(\theta) ]| \le 2\E_\theta[\widetilde z(\theta)^2]$. Applying these bounds to \eqref{eq:G-alpha-second-derivative} and dropping its nonpositive term gives \eqref{eq:G-alpha-curvature-bound}.
\end{proof}

We finally use the preceding calculation to prove the two properties stated in Lemma~\ref{lem:correction-potential-concavity}: concavity of the corrected potential along the alternating path and identification of its left derivative with the played distribution.

\begin{proof}[Proof of Lemma~\ref{lem:correction-potential-concavity}]
For $\alpha=2$, Proposition~\ref{prop:parametric-corrected-potential} gives $G_2''(\theta)\le(-3\eta+16\eta^2+6\eta^3) \Var_{q[\eta R(\theta)]}(z)$. At $\eta=1/10$, the coefficient is negative, so $G_2$ is concave on $[0,1]$.

It remains to identify the derivative at $\theta=0$. Setting $\alpha=2$ in \eqref{eq:G-alpha-first-derivative} gives
\begin{align*}
G_2'(\theta)
&=
-\E_\theta[z]
+4\eta
\E_\theta[
  \widetilde z(\theta)\widetilde\ell(\theta)
]
+2\eta^2
\E_\theta[
  \widetilde z(\theta)\widetilde\ell(\theta)^2
].
\end{align*}
Using $\widetilde z_i(\theta)=z_i-\E_\theta[z]$ and collecting the coefficient of each $z_i$, we can rewrite this as
\begin{align*}
G_2'(\theta)
&=
-\sum_{i=1}^d q[\eta R(\theta)]_i
\left[
1-4\eta\widetilde\ell_i(\theta)
-2\eta^2\left(
  \widetilde\ell_i(\theta)^2
  -\E_\theta[\widetilde\ell(\theta)^2]
\right)
\right]z_i.
\end{align*}
At $\theta=0$, $q[\eta R(0)]=q_{t-1}$, $\widetilde\ell_i(0)=\widetilde\ell_{t-1,i}$, and $\E_0[\widetilde\ell(0)^2]=H_{t-1}$. Therefore $G_2'(0)=-\langle p_t,z\rangle$ by \eqref{eq:played-distribution}, which proves the lemma.
\end{proof}

\section{Omitted Details in Section~\ref{sec:oco-lower}}
\label{app:oco-lower}

We first establish the lower bound on a two-dimensional disk. The construction uses epochs at geometrically decreasing spatial scales. In each epoch, the learner incurs a constant expected loss gap relative to a reference point determined by the losses observed up to that epoch. The reference point continues to move in later epochs, but these future movements form a summable series and therefore change the loss on any earlier epoch by only a small constant. Since the epoch lengths grow geometrically, a horizon of length $T$ contains $\Theta(\log T)$ such scales.

\begin{lemma}[Two-dimensional disk lower bound]
\label{lem:oco-disk-lower}
Let $ \mathcal D=\{x\in\R^2:\|x\|_2\le1\}. $ There exists a universal constant $c_0>0$ such that, for every $T\ge 5$, there is an oblivious distribution $\mathsf P_T$ over sequences of convex losses $f_1,\ldots,f_T:\mathcal D\to[0,1]$ for which every learner satisfies
\begin{align}
  \E_{f_{1:T}\sim\mathsf P_T}\bigl[\AltReg_T\bigr]
  \ge c_0\log T.
  \label{eq:disk-lower-stochastic}
\end{align}
If the learner is randomized, the expectation also includes its internal randomness.
\end{lemma}

\begin{proof}
We first describe the multiscale construction. Fix $ w_1=e_1=(1,0)\in\partial\mathcal D $ and, for $\alpha\in\R$, let $ M_\alpha \triangleq
  \begin{bmatrix}
    \cos\alpha & -\sin\alpha\\
\sin\alpha & \cos\alpha
  \end{bmatrix}
$ be the counterclockwise rotation matrix. For $w\in\partial\mathcal D$ and $\delta\in(0,1/8]$, define two loss functions
\begin{align}
  \mathsf H^{(1)}_w(x)
  &\triangleq \frac{1-\langle w,x\rangle}{2},
  \\
  \mathsf H^{(2)}_{w,\delta}(x)
  &\triangleq
  \frac12\max\left\{
    1-\frac{1-\langle w,x\rangle}{\delta},0
  \right\}.
  \label{eq:disk-hard-losses}
\end{align}
The first function is affine, while the second is the maximum of two affine functions. Hence both are convex. Moreover, since $0\le 1-\langle w,x\rangle\le2$ on $\mathcal D$, both losses take values in $[0,1]$.

In our instance construction, we use several variables defined as follows:
\begin{align}
  \delta_1\triangleq \frac18,
  ~~r\triangleq \frac1{4096},
  ~~\delta_{j+1}&\triangleq r\delta_j,~~p_j\triangleq \frac{\delta_j}{1+\delta_j},
  ~~n_j\triangleq \left\lfloor\frac{1}{2p_j}\right\rfloor,
  ~~\alpha_j\triangleq \arccos(1-\delta_j).
  \label{eq:disk-scale-parameters}
\end{align}
This choice makes the cost of later movements in \eqref{eq:disk-linking-cost} smaller than $3/80$, while the gap in \eqref{eq:disk-current-reference-constant-gap} exceeds $1/20$. Let $\tau_0=0$ and $ \tau_j\triangleq\sum_{k=1}^j(n_k+1). $ Epoch $j$ consists of the $n_j$ rounds $ \mathcal I_j\triangleq\{\tau_{j-1}+1,\ldots,\tau_{j-1}+n_j\} $ followed by the zero-loss separator $f_{\tau_j}\equiv0$, which separates the paired-loss contributions of consecutive epochs.

Now we define our loss function at each epoch. Conditioned on the current boundary point $w_j$, the losses in epoch $\mathcal I_j$ are sampled independently according to the following probability:
\begin{align}
  f_t=
  \begin{cases}
    \mathsf H^{(2)}_{w_j,\delta_j}, & \text{with probability }p_j,\\
    \mathsf H^{(1)}_{w_j}, & \text{with probability }1-p_j.
  \end{cases}
  \label{eq:disk-epoch-distribution}
\end{align}
Let $N_j$ be the number of $\mathsf{H}^{(2)}_{w_j,\delta_j}$ loss functions in epoch $j$. At the end of the epoch, we rotate $w_j$ by $\alpha_j$ if $N_j\geq 1$ and keep the direction otherwise. Specifically, we define $w_{j+1}$ as follows:
\begin{align}
  w_{j+1}
  \triangleq
  \begin{cases}
    w_j, & N_j=0,\\
    M_{\alpha_j}w_j, & N_j\ge1.
  \end{cases}
  \label{eq:disk-center-update}
\end{align}
All Bernoulli variables used in \eqref{eq:disk-epoch-distribution} can be sampled before the interaction begins. Once these variables are fixed, the entire sequence $w_1,w_2,\ldots$ and hence the entire loss sequence are fixed as well. Thus the resulting random environment is oblivious.

We next compare the learner's loss with the reference points generated by this construction. For epoch $j$, define the learner's paired loss $L_j\triangleq\sum_{t\in\mathcal I_j} \bigl(f_t(x_t)+f_t(x_{t+1})\bigr) $ and for any $v\in\mathcal D$, define its (doubled) comparator loss on this epoch by $C_j(v) \triangleq 2\sum_{t\in\mathcal I_j}f_t(v).$ The zero-loss separator contributes to neither quantity.

Fix an epoch $j$ and condition on the history before this epoch. For any $\cF_{t-1}$-measurable action $x_t\in\mathcal D$, let $ s=1-\langle w_j,x_t\rangle. $ If $s\le\delta_j$, then the choice $p_j=\delta_j/(1+\delta_j)$ gives that
\begin{align*}
  \E[f_t(x_t)\mid\cF_{t-1}]
  =p_j\mathsf H^{(2)}_{w_j,\delta_j}(x_t)
    +(1-p_j)\mathsf H^{(1)}_{w_j}(x_t)=\frac{p_j}{2}
    +\frac{s}{2}\left(1-p_j-\frac{p_j}{\delta_j}\right)
  =\frac{p_j}{2}.
\end{align*}
If $s\ge\delta_j$, then the $\mathsf H^{(1)}_{w_j}$ loss alone gives $ \E[f_t(x_t)\mid\cF_{t-1}] \ge \frac{(1-p_j)\delta_j}{2} =\frac{p_j}{2}. $ Since every loss is nonnegative, $f_t(x_{t+1})\ge0$. Summing over the epoch yields
\begin{align}
  \E[L_j]\ge\frac{n_jp_j}{2}.
  \label{eq:disk-learner-epoch-loss}
\end{align}

Now define the reference point after epoch $j$ by $ v_j\triangleq w_{j+1}. $ If $N_j=0$, then $v_j=w_j$ and every loss in the epoch is $\mathsf H^{(1)}_{w_j}$, so $C_j(v_j)=0$. If $N_j\ge1$, then $ v_j=M_{\alpha_j}w_j $ and, by the definition of $\alpha_j$, we know that
\begin{align*}
  1-\langle w_j,v_j\rangle
  =1-\cos\alpha_j
  =\delta_j.
\end{align*}
Hence $\mathsf{H}^{(2)}_{w_j,\delta_j}(v_j)=0$, while each $\mathsf{H}^{(1)}_{w_j}(v_j)$ contributes $\delta_j$ to the doubled comparator loss. Therefore, pathwise, we have that
\begin{align}
  C_j(v_j)
  =(n_j-N_j)\delta_j\mathbf 1\{N_j\ge1\}.
  \label{eq:disk-reference-own-epoch}
\end{align}
Let $ a_j\triangleq(1-p_j)^{n_j}. $ Since $N_j\sim\mathrm{Binomial}(n_j,p_j)$, \eqref{eq:disk-reference-own-epoch} shows that
\begin{align*}
  \E[C_j(v_j)]
  &= \E[n_j\delta_j(1-a_j) - \delta_jN_j] = n_j\delta_j\bigl((1-p_j)-a_j\bigr).
\end{align*}
Combining the above identity with \eqref{eq:disk-learner-epoch-loss} and using $\delta_j=p_j/(1-p_j)$ gives
\begin{align}
  \E[L_j-C_j(v_j)]\geq \frac{n_jp_j}{2}-n_j\delta_j\bigl((1-p_j)-a_j\bigr)
  =
  \frac{n_jp_j}{2}
  \left(\frac{2a_j}{1-p_j}-1\right).
  \label{eq:disk-current-reference-gap}
\end{align}
Because $\delta_j\le1/8$, we have $p_j\le1/9$. Moreover,
\begin{align*}
  n_jp_j
  &=p_j\left\lfloor\frac{1}{2p_j}\right\rfloor
  \in\left[\frac12-p_j,\frac12\right]
  \subseteq\left[\frac13,\frac12\right].
\end{align*}
Also, using $n_jp_j\le1/2$ and $ \log(1-p)\ge-p/(1-p) $ for $p\in(0,1)$, we know that
\begin{align*}
  \frac{a_j}{1-p_j}
  =(1-p_j)^{n_j-1}
  \ge\exp\left(-\frac{(n_j-1)p_j}{1-p_j}\right)
  \ge e^{-n_jp_j}.
\end{align*}
Since $u\mapsto u(2e^{-u}-1)/2$ is decreasing on $[1/3,1/2]$, for any $j\geq 1$,
\begin{align}
  \E[L_j-C_j(v_j)]
  \ge
  \frac14\bigl(2e^{-1/2}-1\bigr)
  >\frac1{20}.
  \label{eq:disk-current-reference-constant-gap}
\end{align}
Thus, in every epoch, the learner's expected paired loss exceeds the loss of the reference point $v_j$ by a universal constant.

Next, we bound the expected loss of the reference point $v_j$ in each epoch $k$ that is before $j$. Fix $k\le j$ and define the angle between $v_k$ and $v_j$ to be $\beta_{k\leftrightarrow j}$. Since every update in \eqref{eq:disk-center-update} is a forward rotation, $\beta_{k\leftrightarrow j}$ is at most
\begin{align}\label{eqn:beta}
  \beta_{k\leftrightarrow j}\leq
  \sum_{\ell=k+1}^{\infty}\alpha_\ell
  =\sum_{\ell=k+1}^{\infty}\arccos(1-\delta_\ell)
  \leq \sum_{\ell=k+1}^{\infty}\sqrt{3\delta_\ell}
  =\frac{\sqrt{3r\delta_k}}{1-\sqrt r}
  \leq \frac{4}{3}\sqrt{3r\delta_k},
\end{align}
where the first equality is by the definition of $\alpha_\ell$, the second inequality uses the elementary bound $\arccos(1-\delta)\le\sqrt{3\delta}$ for $\delta\in(0,\frac18]$, the second equality follows from $\delta_\ell=r^{\ell-k}\delta_k$, and the final inequality uses $1/(1-\sqrt r)\le4/3$, which holds since $r=1/4096$.

We now bound how much the later movement from $v_k$ to $v_j$ can increase the loss of $v_j$ on epoch $k$. If $N_k=0$, then $v_k=w_k$ and all losses in the epoch are $\mathsf{H}^{(1)}_{w_k}$. Therefore, using the fact that $1-\cos z\le z^2/2$ gives that
\begin{align}
  C_k(v_j)-C_k(v_k)
  =2n_k\cdot \left(\mathsf{H}^{(1)}_{w_k}(v_j) - \mathsf{H}^{(1)}_{w_k}(v_k)\right)
  =n_k(1-\cos\beta_{k\leftrightarrow j})
  \le\frac{n_k\beta_{k\leftrightarrow j}^2}{2}\leq \frac{8rn_k\delta_k}{3}
  \le\frac{3}{2} r,
  \label{eq:disk-link-no-cap}
\end{align}
where the second inequality uses \eqref{eqn:beta} and the last inequality is because $n_k \le\frac{9}{16\delta_k}$ as $ n_k=\lfloor(1+\delta_k)/(2\delta_k)\rfloor $ and $\delta_k\le1/8$.

If $N_k\ge1$, then $v_k=M_{\alpha_k}w_k$. Since all subsequent rotations are in the same direction and $\alpha_k+\beta_{k\leftrightarrow j}<\pi$, we know that
\begin{align*}
    1-\inner{w_k}{v_j}\geq 1-\inner{w_k}{v_k} =\delta_k,
\end{align*}
meaning that every $\mathsf{H}^{(2)}_{w_k,\delta_k}$ loss function remains zero at $v_j$. Here, $\alpha_k+\beta_{k\leftrightarrow j}<\pi$ follows from $\alpha_k\le\sqrt{3\delta_k}$, \eqref{eqn:beta}, $\delta_k\le1/8$, and $r=1/4096$. Thus only the $\mathsf{H}^{(1)}_{w_k}$ losses can increase. Using the fact that $ \cos\alpha-\cos(\alpha+\beta) \le\beta(\alpha+\beta) $ for $\alpha,\beta\ge0$, together with \eqref{eqn:beta}, we obtain that
\begin{align}
  C_k(v_j)-C_k(v_k)
  &=(n_k-N_k)\inner{w_k}{v_k-v_j}\nonumber\\
  &\le
  n_k\bigl(\cos\alpha_k-\cos(\alpha_k+\beta_{k\leftrightarrow j})\bigr)\nonumber\\
  &\le n_k\beta_{k\leftrightarrow j}(\alpha_k+\beta_{k\leftrightarrow j})
  \le\frac94\sqrt r+3r,
  \label{eq:disk-link-cap}
\end{align}
where the last inequality uses $\alpha_k\le\sqrt{3\delta_k}$, $n_k\le 9/(16\delta_k)$, $r=1/4096$, and \eqref{eqn:beta}. Combining \eqref{eq:disk-link-no-cap} and \eqref{eq:disk-link-cap}, and again using $r=1/4096$, gives the pathwise bound
\begin{align}
  C_k(v_j)-C_k(v_k)
  \le\frac94\sqrt r
  +3r
  <\frac{3}{80}.
  \label{eq:disk-linking-cost}
\end{align}
Further combining \eqref{eq:disk-current-reference-constant-gap} with \eqref{eq:disk-linking-cost} shows that for every $k\le j$,
\begin{align}
  \E[L_k-C_k(v_j)]
  \ge\E[L_k-C_k(v_k)] + \E[C_k(v_k)-C_k(v_j)]\geq \frac{1}{20}-\frac{3}{80} = \frac{1}{80}.
  \label{eq:disk-final-reference-gap}
\end{align}

We now relate this calculation to the actual hindsight optimizer for the entire prefix. As
\begin{align*}
  2\min_{u\in\mathcal{D}}\sum_{t=1}^{\tau_j}f_t(u)
  \le
  2\sum_{t=1}^{\tau_j}f_t(v_j)
  =\sum_{k=1}^j C_k(v_j),
\end{align*}
and the learner's paired loss through round $\tau_j$ is $\sum_{k=1}^jL_k$ because all separators are zero, we know that
\begin{align}
  \AltReg_{\tau_j}
  &\ge
  \sum_{k=1}^j\bigl(L_k-C_k(v_j)\bigr).
  \label{eq:disk-prefix-regret}
\end{align}
Taking expectations and applying \eqref{eq:disk-final-reference-gap} yields
\begin{align}
  \E[\AltReg_{\tau_j}]
  \ge\frac{j}{80}.
  \label{eq:disk-j-epochs-regret}
\end{align}
It remains to count how many complete epochs fit in the horizon. From the definition of $n_k$ and the fact that $\delta_k\le1/8$, we have $n_k+1\le\delta_k^{-1}$. Hence, the total length of the first $j$ epochs satisfies
\begin{align*}
  \tau_j
  =\sum_{k=1}^j(n_k+1)
  \le\sum_{k=1}^j\delta_k^{-1}
  \le\frac{2}{\delta_j},
\end{align*}
where the last inequality uses $\delta_k=\delta_jr^{k-j}$ and $r<1/2$.

In particular, since $\tau_1=5$, for every $T\ge5$, let $J$ be the largest index satisfying $\tau_J\le T$. Then, we have
\begin{align*}
  T
  <\tau_{J+1}
  \le\frac{2}{\delta_{J+1}}
  =16r^{-J}.
\end{align*}
Taking logarithms and using $J\ge1$ gives
\begin{align}
  \log T
  &\le \log 16+J\log(1/r)
  \le J\log(16/r),
  \label{eq:disk-count-scales}
\end{align}
meaning that $ J\ge\frac{\log T}{\log(16/r)}. $ We then run these $J$ complete epochs and set all remaining losses to zero. Since the additional zero losses do not change the alternating regret, \eqref{eq:disk-j-epochs-regret} and \eqref{eq:disk-count-scales} imply that
\begin{align*}
  \E[\AltReg_T]
  \ge
  \frac{J}{80}
  \ge
  \frac{1}{80\log(16/r)}
  \log T.
\end{align*}
\end{proof}

We are now ready to prove Theorem~\ref{thm:oco-lower}. The proof lifts Lemma~\ref{lem:oco-disk-lower} to higher dimensions by a direct-product construction. For horizons that are too short to invoke the disk lemma on every two-dimensional coordinate block, we use a separate coordinatewise construction.

\begin{proof}[Proof of Theorem~\ref{thm:oco-lower}]
Fix $d\ge2$ and $T>2$, and let $ m\triangleq\lfloor d/2\rfloor. $ We first construct an oblivious distribution over loss sequences and then use an averaging argument to obtain a single deterministic loss sequence. We consider two regimes.

\paragraph{Long-horizon regime: $T\ge5m$.}
Recall that $ \mathcal D=\{z\in\R^2:\|z\|_2\le1\}. $ We choose the $d$-dimensional set as follows
\begin{align*}
  \cX
  =
  \begin{cases}
    \mathcal D^m, & d=2m,\\
    \mathcal D^m\times[-1,1], & d=2m+1.
  \end{cases}
\end{align*}
Equivalently, for every $x\in\cX$ and $k\in[m]$, the coordinate pair $ (x_{2k-1},x_{2k}) $ belongs to $\mathcal D$; when $d$ is odd, the last coordinate $x_d$ belongs to $[-1,1]$ and will not be used.

Let $ H\triangleq\lfloor T/m\rfloor. $ Since $T\ge5m$, we have $H\ge5$. Therefore, by Lemma~\ref{lem:oco-disk-lower}, there exists an oblivious distribution $\mathsf P_H$ over sequences $ (g_1,\ldots,g_H) $ of convex loss functions $g_s:\mathcal D\to[0,1]$ such that every learner on $\mathcal D$ satisfies
\begin{align}
  \E_{g_{1:H}\sim\mathsf P_H}
  \left[
    \sum_{s=1}^H
    \bigl(g_s(z_s)+g_s(z_{s+1})\bigr)
    -
    2\min_{z\in\mathcal D}\sum_{s=1}^H g_s(z)
  \right]
  \ge c_0\log H.
  \label{eq:oco-disk-instance-H}
\end{align}
Fix one such distribution $\mathsf P_H$. We then partition the first $mH$ rounds into $m$ consecutive segments, where segment $k\in[m]$ consists of rounds
\begin{align*}
  \mathcal J_k
  \triangleq
  \{(k-1)H+1,\ldots,kH\}.
\end{align*}
Independently for every $k\in[m]$, sample a sequence $ (g_{k,1},\ldots,g_{k,H})\sim\mathsf P_H, $ where each $g_{k,s}:\mathcal D\to[0,1]$. All these sequences are sampled before the interaction begins without the knowledge of the learner's decision. For the $s$-th round of segment $k$, namely $t=(k-1)H+s$, define
\begin{align}
  f_t(x)
  \triangleq
  g_{k,s}\bigl(x_{2k-1},x_{2k}\bigr).
  \label{eq:oco-product-loss}
\end{align}
For all remaining rounds $t>mH$, set $f_t\equiv0$. Since $(x_{2k-1},x_{2k})\in\mathcal D$ for every $x\in\cX$, each $f_t$ is well defined, convex, and takes values in $[0,1]$.

We next show that the alternating regret decomposes across the $m$ segments. For every $k\in[m]$ and $s\in[H+1]$, define
\begin{align*}
  z_{k,s}
  \triangleq
  \bigl(
    x_{(k-1)H+s,\,2k-1},
    x_{(k-1)H+s,\,2k}
  \bigr)
  \in\mathcal D.
\end{align*}
Thus $z_{k,s}$ is the pair of coordinates $(2k-1,2k)$ used by the learner at the corresponding time. For each realized sequence, let
\begin{align*}
  z_k^\star
  \in
  \arg\min_{z\in\mathcal D}
  \sum_{s=1}^H g_{k,s}(z).
\end{align*}
Define a single comparator $u^\star\in\cX$ where $(u^\star_{2k-1},u^\star_{2k}) =z_k^\star$ for all $k\in[m]$ and, when $d$ is odd, set $u^\star_d=0$. Since the losses in segment $k$ depend only on coordinates $(2k-1,2k)$, the minimization over the product domain separates, and hence we have
\begin{align}
  \min_{u\in\cX}\sum_{t=1}^T f_t(u)
  =
  \sum_{k=1}^m
  \min_{z\in\mathcal D}\sum_{s=1}^H g_{k,s}(z).
  \label{eq:oco-product-comparator}
\end{align}
Together with \eqref{eq:oco-product-loss}, this gives
\begin{align}
  \AltReg_T
  =
  \sum_{k=1}^m
  \left[
    \sum_{s=1}^H
    \bigl(
      g_{k,s}(z_{k,s})
      +
      g_{k,s}(z_{k,s+1})
    \bigr)
    -
    2\min_{z\in\mathcal D}
    \sum_{s=1}^H g_{k,s}(z)
  \right],
  \label{eq:oco-product-regret-decomposition}
\end{align}
where the losses after round $mH$ are zero and therefore make no contribution.

To lower bound the expectation of each term in \eqref{eq:oco-product-regret-decomposition}, fix any segment $k$ and condition on the information available before this segment, namely $\cF_{(k-1)H}$. During segment $k$, $(g_{k,1},\ldots,g_{k,H})$ is independent of $\cF_{(k-1)H}$ and is distributed according to $\mathsf P_H$. Therefore, Lemma~\ref{lem:oco-disk-lower}, applied conditionally on $\cF_{(k-1)H}$, yields
\begin{align*}
  \E\left[
    \left.
    \sum_{s=1}^H
    \bigl(
      g_{k,s}(z_{k,s})
      +
      g_{k,s}(z_{k,s+1})
    \bigr)
    -
    2\min_{z\in\mathcal D}
    \sum_{s=1}^H g_{k,s}(z)
    \,\right|\,
    \cF_{(k-1)H}
  \right]
  \ge
  c_0\log H.
\end{align*}
Taking expectations and summing over $k\in[m]$ in \eqref{eq:oco-product-regret-decomposition} gives
\begin{align}
  \E[\AltReg_T]
  \ge
  c_0m\log H.
  \label{eq:oco-product-intermediate}
\end{align}
Finally, for every $d\ge2$, we have $ m=\lfloor d/2\rfloor\ge d/3. $ Since $T/m\ge5$, we have $ H =\left\lfloor\frac{T}{m}\right\rfloor \ge\frac{T}{2m} \ge\frac{T}{d}. $ Moreover, $H\ge5$ implies $ \log H\ge\frac12\log\left(1+\frac{T}{d}\right). $ Combining these inequalities with \eqref{eq:oco-product-intermediate} yields $ \E[\AltReg_T] \ge \frac{c_0}{6} d\log\left(1+\frac{T}{d}\right). $

\paragraph{Short-horizon regime: $T<5m$.} When $T<5m$, the instance construction is from the standard OCO lower bound.
Choose $\cX=[-1,1]^d$ and let $K\triangleq\min\{T,d\}$. Draw independent Rademacher signs $\sigma_1,\ldots,\sigma_K$ before play and define $ f_t(x) \triangleq \frac{1+\sigma_t x^{(t)}}{2}$ for all $t\in[K]$, where $x^{(t)}$ denotes the $t$-th coordinate of $x$, distinguishing it from the iterate $x_t$, and $f_t\equiv0$ for $t>K$.

Since $x_t$ is chosen before $\sigma_t$ is revealed, we have $ \E[f_t(x_t)\mid\cF_{t-1}]=1/2, $ while $f_t(x_{t+1})\ge0$. On the other hand, define the fixed comparator $u^\star\in[-1,1]^d$ as $u^\star_i=-\sigma_i$ for $i\in[K]$ and $u^\star_i=0$ for $i>K$. In this case, $f_t(u^\star)+f_{t+1}(u^\star)=0$ for all $t$. Hence $ \E[\AltReg_T]\ge\frac{K}{2}.$ Moreover, as $T<5m\le5d/2$, we have
\begin{align*}
  K
  =d\min\left\{\frac{T}{d},1\right\}
  \ge
  \frac{d}{2}
  \log\left(1+\frac{T}{d}\right),
\end{align*}
where we use $ \min\{x,1\}\ge\log(1+x)/2 $ for $x\in[0,5/2]$.

Thus, in both regimes, there exists a universal constant $ c\triangleq \min\left\{ \frac{c_0}{6}, \frac{1}{4} \right\}>0 $ and an oblivious distribution over deterministic loss sequences such that $ \E[\AltReg_T] \ge cd\log\left(1+\frac{T}{d}\right). $ Since the expectation is an average over deterministic loss sequences, for the fixed learner there exists at least one sequence in the support whose expected alternating regret is at least the same quantity. This finishes the proof.
\end{proof}

\section{Omitted Details in Section~\ref{sec:oco-upper}}
\label{app:oco-upper-details}

This appendix verifies the feasibility of the continuous correction and proves the two auxiliary lemmas used in the proof of Theorem~\ref{thm:oco-upper}. The corrected-potential calculation is the continuous counterpart of Proposition~\ref{prop:parametric-corrected-potential}. We therefore reuse that calculation whenever possible and only verify the additional analytic step needed to pass from finite expectations to integrals.

We first verify that the corrected density in \eqref{eq:continuous-played-density} is well defined and record the rejection sampling representation discussed in Section~\ref{sec:oco-upper}.

\begin{proposition}
\label{prop:continuous-feasibility}
For $\eta=1/10$, every density $p_t$ defined in \eqref{eq:continuous-played-density} belongs to $\Delta_{\cX}$. Moreover, if $ \Gamma_t(x)\triangleq 1-4\eta \widetilde f_{t-1}(x) -2\eta^2\bigl(\widetilde f_{t-1}(x)^2-H_{t-1}\bigr), $ then $0<\Gamma_t(x)<2$ for every $x\in\cX$.
\end{proposition}

\begin{proof}
By the definitions of $\bar f_{t-1}$ and $H_{t-1}$,
\begin{align*}
\int_{\cX}\widetilde f_{t-1}(x)q_{t-1}(x)\,dx&=0,\\
\int_{\cX}(\widetilde f_{t-1}(x)^2-H_{t-1})q_{t-1}(x)\,dx&=0.
\end{align*}
Hence $\int_{\cX}p_t(x)\,dx=1$. Since $f_{t-1}(x)\in[-1,1]$, $|\widetilde f_{t-1}(x)|\le2$ and $H_{t-1}\le1$. Thus, at $\eta=1/10$, $\Gamma_t(x)\ge1-8\eta-8\eta^2>0$ and $\Gamma_t(x)\le1+8\eta+2\eta^2<2$. Therefore $p_t$ is a probability density.
\end{proof}

\begin{remark}
\label{rem:continuous-rejection}
Proposition~\ref{prop:continuous-feasibility} also gives a simple way to sample from $p_t$. Recall that $ p_t(x)=q_{t-1}(x)\Gamma_t(x) $ and $0<\Gamma_t(x)<2$. Given the exact values of $\bar f_{t-1}$ and $H_{t-1}$, draw $X\sim q_{t-1}$ and accept it with probability $\Gamma_t(X)/2$. The density of an accepted proposal is proportional to $ q_{t-1}(x)\Gamma_t(x)/2=p_t(x)/2. $ Moreover, $ \E_{q_{t-1}}[\Gamma_t]=\int_{\cX}p_t(x)\,dx=1, $ so the overall acceptance probability is $1/2$. Therefore, conditioned on acceptance, $X$ has density $p_t$, and an exact sample from $p_t$ requires two proposals from $q_{t-1}$ in expectation.
\end{remark}

We next prove Lemma~\ref{lem:continuous-corrected-potential}. The only new step relative to Appendix~\ref{app:expert-potential} is to verify the moving-expectation identity for the continuous exponential-weights density. Once this identity is established, the derivative calculation in Proposition~\ref{prop:parametric-corrected-potential} applies without change.

\begin{proof}[Proof of Lemma~\ref{lem:continuous-corrected-potential}]
Fix $t\in[T+1]$ and define $R(\theta)\triangleq R_{t-1}-\theta h_t$ and $f(\theta)\triangleq f_{t-1}-\theta h_t$ for $\theta\in[0,1]$. The loss path $f(\theta)=(1-\theta)f_{t-1}-\theta f_t$ takes values in $[-1,1]$. Let $q_\theta\triangleq q[\eta R(\theta)]$. Since $R(\theta)$, $h_t$, and $f(\theta)$ are uniformly bounded over $\theta\in[0,1]$ and $x\in\cX$, differentiation under the integral is justified. Differentiating the exponential-weights density gives
\begin{align*}
\nabla_\theta q_\theta(x)
&=
-\eta q_\theta(x)
\left(
h_t(x)-\E_{q_\theta}[h_t]
\right).
\end{align*}
Consequently, for any bounded differentiable family $\phi_\theta$ with $\nabla_\theta\phi_\theta(x)$ uniformly bounded over $(\theta,x)$, we have
\begin{align}
\frac{\mathrm d}{\mathrm d\theta}
\E_{q_\theta}[\phi_\theta]
&=
\E_{q_\theta}[\nabla_\theta\phi_\theta]
-\eta
\E_{q_\theta}\left[
\left(
h_t-\E_{q_\theta}[h_t]
\right)\phi_\theta
\right].
\label{eq:continuous-moving-expectation}
\end{align}

Equation~\eqref{eq:continuous-moving-expectation} is exactly the continuous counterpart of \eqref{eq:moving-expectation}. To make the correspondence with Proposition~\ref{prop:parametric-corrected-potential} explicit, the objects in that proposition are replaced as follows: $z$ is replaced by the function $h_t$, $q[\eta R(\theta)]$ is replaced by the density $q_\theta$, finite expectations are replaced by integrals under $q_\theta$, and $\widetilde z_i(\theta)$ is replaced by $ \widetilde h_\theta(x) \triangleq h_t(x)-\E_{q_\theta}[h_t]. $ Similarly, define $ \widetilde f_\theta(x) \triangleq f(\theta,x)-\E_{q_\theta}[f(\theta)]. $

Every subsequent step in the proof of Proposition~\ref{prop:parametric-corrected-potential} uses only the moving-expectation identity, differentiation of the affine interpolation, and algebraic identities for centered quantities. Therefore the same calculation gives, for $ G_\alpha(\theta) \triangleq F(R(\theta)) -\alpha\eta\Var_{q_\theta}(f(\theta)), $
\begin{align*}
G_\alpha'(\theta)
&=
-\E_{q_\theta}[h_t]
+2\alpha\eta
\E_{q_\theta}[
  \widetilde h_\theta\widetilde f_\theta
]
+\alpha\eta^2
\E_{q_\theta}[
  \widetilde h_\theta\widetilde f_\theta^2
],
\end{align*}
and
\begin{align*}
G_\alpha''(\theta)
&\le
\left[
(1-2\alpha)\eta
+8\alpha\eta^2
+3\alpha\eta^3
\right]
\Var_{q_\theta}(h_t).
\end{align*}
These are precisely \eqref{eq:G-alpha-first-derivative} and \eqref{eq:G-alpha-curvature-bound}, respectively, with the above continuous substitutions. Setting $\alpha=2$ and $\eta=1/10$ gives $G_2''(\theta)\le0$, so $G_2$ is concave.

The interpolation also has the required endpoints. Indeed, $R(1)=R_{t-1}-h_t=R_t$ and $f(1)=f_{t-1}-h_t=-f_t$. Since variance is unchanged by negation, we have $G_2(0)=U(R_{t-1},f_{t-1})$ and $G_2(1)=U(R_t,-f_t)\allowbreak=U(R_t,f_t)$. The same identities hold for $t=T+1$ under the endpoint convention in Section~\ref{sec:oco-upper}.

It remains to identify the left derivative. Setting $\alpha=2$ in the first-derivative formula above gives
\begin{align*}
G_2'(\theta)
&=
-\E_{q_\theta}[h_t]
+4\eta
\E_{q_\theta}[
  \widetilde h_\theta\widetilde f_\theta
]
+2\eta^2
\E_{q_\theta}[
  \widetilde h_\theta\widetilde f_\theta^2
].
\end{align*}
As in the last step of the proof of Lemma~\ref{lem:correction-potential-concavity}, collecting the coefficient of $h_t(x)$ rewrites this derivative as
\begin{align*}
G_2'(\theta)
&=
-\int_{\cX}
q_\theta(x)
\left[
1
-4\eta\widetilde f_\theta(x)
-2\eta^2
\left(
\widetilde f_\theta(x)^2
-\E_{q_\theta}[\widetilde f_\theta^2]
\right)
\right]
h_t(x)\,dx.
\end{align*}
At $\theta=0$, we have $q_\theta=q_{t-1}$, $\widetilde f_\theta(x)=\widetilde f_{t-1}(x)$, and $\E_{q_{t-1}}[\widetilde f_\theta^2]=H_{t-1}$. Hence, by \eqref{eq:continuous-played-density}, $ G_2'(0) = -\int_{\cX}p_t(x)h_t(x)\,dx. $ Concavity therefore gives $ G_2(1)-G_2(0)\le G_2'(0), $ which is exactly \eqref{eq:continuous-corrected-potential-inequality}.
\end{proof}

Finally, we prove Lemma~\ref{lem:continuous-comparator-bound}. Unlike the corrected-potential calculation above, this step is genuinely new in the continuous setting. A single comparator has zero volume, so we compare it with a scaled copy of the domain around it.

\begin{proof}[Proof of Lemma~\ref{lem:continuous-comparator-bound}]
Fix $u\in\cX$ and $\delta\in(0,1)$, and let $A_u\triangleq(1-\delta)u+\delta\cX$. Since $\cX$ is $d$-dimensional, $\vol(A_u)=\delta^d\vol(\cX)$. For any $a\in\cX$ and $x=(1-\delta)u+\delta a\in A_u$, convexity and the range $f_t\in[-1,1]$ give $f_t(x)\le(1-\delta)f_t(u)+\delta f_t(a)\le f_t(u)+2\delta$. Because $R_{T+1}=-2\sum_{t=1}^T f_t$, it follows that $R_{T+1}(x)\ge R_{T+1}(u)-4T\delta$ for every $x\in A_u$.

Restricting the log-partition integral to $A_u$ gives
\begin{align*}
F(R_{T+1})
\ge
\frac{1}{\eta}
\log\left(
\frac{\vol(A_u)}{\vol(\cX)}
\exp\bigl(
\eta(R_{T+1}(u)-4T\delta)
\bigr)
\right)=
R_{T+1}(u)
-4T\delta
-\frac d\eta\log\frac1\delta,
\end{align*}
which proves \eqref{eq:continuous-comparator-bound}.
\end{proof}

\clearpage
\section*{AI Methodology Statement}
\label{sec:ai-methodology}
The author used GPT-5.6 and GPT-6 Astra for assistance with writing and to explore proof strategies. The author has checked the arguments and takes full responsibility for the content of the paper.

\bibliographystyle{alpha}
\bibliography{references}

@article{tammelin2014solving,
  title={Solving large imperfect information games using {CFR+}},
  author={Tammelin, Oskari},
  journal={arXiv preprint arXiv:1407.5042},
  year={2014}
}

@article{littlestone1994weighted,
  author  = {Nick Littlestone and Manfred K. Warmuth},
  title   = {The Weighted Majority Algorithm},
  journal = {Information and Computation},
  volume  = {108},
  number  = {2},
  pages   = {212--261},
  year    = {1994}
}

@article{freund1997decision,
  author  = {Yoav Freund and Robert E. Schapire},
  title   = {A Decision-Theoretic Generalization of On-Line Learning and an Application to Boosting},
  journal = {Journal of Computer and System Sciences},
  volume  = {55},
  number  = {1},
  pages   = {119--139},
  year    = {1997}
}

@book{cesa2006prediction,
  author    = {Nicol{\`o} Cesa-Bianchi and G{\'a}bor Lugosi},
  title     = {Prediction, Learning, and Games},
  publisher = {Cambridge University Press},
  year      = {2006}
}

@inproceedings{zinkevich2003online,
  author    = {Martin Zinkevich},
  title     = {Online Convex Programming and Generalized Infinitesimal Gradient Ascent},
  booktitle = {Proceedings of the 20th International Conference on Machine Learning},
  pages     = {928--936},
  year      = {2003}
}

@article{hazan2016introduction,
  title={Introduction to online convex optimization},
  author={Hazan, Elad},
  journal={Foundations and Trends in Optimization},
  volume={2},
  number={3-4},
  pages={157--325},
  year={2016},
  publisher={Emerald Publishing Limited}
}

@article{hazan2010variation,
  author  = {Elad Hazan and Satyen Kale},
  title   = {Extracting Certainty from Uncertainty: Regret Bounded by Variation in Costs},
  journal = {Machine Learning},
  volume  = {80},
  number  = {2--3},
  pages   = {165--188},
  year    = {2010}
}

@article{rakhlin2013predictable,
  title={Optimization, learning, and games with predictable sequences},
  author={Rakhlin, Sasha and Sridharan, Karthik},
  journal={Advances in Neural Information Processing Systems},
  volume={26},
  year={2013}
}

@inproceedings{bailey2020finite,
  author    = {James P. Bailey and Gauthier Gidel and Georgios Piliouras},
  title     = {Finite Regret and Cycles with Fixed-Step-Size via Alternating Gradient Descent-Ascent},
  booktitle = {Proceedings of the 33rd Conference on Learning Theory},
  series    = {Proceedings of Machine Learning Research},
  volume    = {125},
  pages     = {391--407},
  year      = {2020}
}

@inproceedings{wibisono2022alternating,
  author    = {Andre Wibisono and Molei Tao and Georgios Piliouras},
  title     = {Alternating Mirror Descent for Constrained Min-Max Games},
  booktitle = {Advances in Neural Information Processing Systems 35},
  year      = {2022}
}

@inproceedings{cevher2023alternation,
  author    = {Volkan Cevher and Ashok Cutkosky and Ali Kavis and Georgios Piliouras and Stratis Skoulakis and Luca Viano},
  title     = {Alternation Makes the Adversary Weaker in Two-Player Games},
  booktitle = {Advances in Neural Information Processing Systems 36},
  year      = {2023}
}

@inproceedings{hait2025alternating,
  author    = {Soumita Hait and Ping Li and Haipeng Luo and Mengxiao Zhang},
  title     = {Alternating Regret for Online Convex Optimization},
  booktitle = {Proceedings of the 38th Annual Conference on Learning Theory},
  series    = {Proceedings of Machine Learning Research},
  volume    = {291},
  pages     = {2632--2633},
  year      = {2025}
}

@article{freund1999adaptive,
  author  = {Yoav Freund and Robert E. Schapire},
  title   = {Adaptive Game Playing Using Multiplicative Weights},
  journal = {Games and Economic Behavior},
  volume  = {29},
  number  = {1--2},
  pages   = {79--103},
  year    = {1999}
}

@article{burch2019revisiting,
  author  = {Neil Burch and Matej Morav{\v c}{\'i}k and Martin Schmid},
  title   = {Revisiting {CFR+} and Alternating Updates},
  journal = {Journal of Artificial Intelligence Research},
  volume  = {64},
  pages   = {429--443},
  year    = {2019}
}

@inproceedings{syrgkanis2015fast,
  author    = {Vasilis Syrgkanis and Alekh Agarwal and Haipeng Luo and Robert E. Schapire},
  title     = {Fast Convergence of Regularized Learning in Games},
  booktitle = {Advances in Neural Information Processing Systems 28},
  year      = {2015}
}

@inproceedings{chen2020hedging,
  author    = {Xi Chen and Binghui Peng},
  title     = {Hedging in Games: Faster Convergence of External and Swap Regrets},
  booktitle = {Advances in Neural Information Processing Systems 33},
  pages     = {18990--18999},
  year      = {2020}
}

@inproceedings{daskalakis2021nearoptimal,
  author    = {Constantinos Daskalakis and Maxwell Fishelson and Noah Golowich},
  title     = {Near-Optimal No-Regret Learning in General Games},
  booktitle = {Advances in Neural Information Processing Systems 34},
  pages     = {27604--27616},
  year      = {2021}
}

@inproceedings{farina2022nearoptimal,
  author    = {Gabriele Farina and Ioannis Anagnostides and Haipeng Luo and Chung-Wei Lee and Christian Kroer and Tuomas Sandholm},
  title     = {Near-Optimal No-Regret Learning Dynamics for General Convex Games},
  booktitle = {Advances in Neural Information Processing Systems 35},
  pages     = {39076--39089},
  year      = {2022}
}

@article{katona2026symplectic,
  author  = {Jonas E. Katona and Xiuyuan Wang and Andre Wibisono},
  title   = {A Symplectic Analysis of Alternating Mirror Descent},
  journal = {Journal of Machine Learning Research},
  volume  = {27},
  number  = {44},
  pages   = {1--61},
  year    = {2026}
}

@inproceedings{lazarsfeld2025optimism,
  author    = {John Lazarsfeld and Georgios Piliouras and Ryann Sim and Stratis Skoulakis},
  title     = {Optimism Without Regularization: Constant Regret in Zero-Sum Games},
  booktitle = {Advances in Neural Information Processing Systems 38},
  year      = {2025}
}

@article{tao2026optimal,
  author  = {Yixin Tao and Weiqiang Zheng},
  title   = {Optimal Alternating Regret for Online Learning and Games},
  journal = {arXiv preprint arXiv:2608.24731},
  year    = {2026}
}

@article{liu2026constant,
  author  = {Mingyang Liu and Gabriele Farina and Asuman Ozdaglar},
  title   = {Constant Individual Regret in General Games},
  journal = {arXiv preprint arXiv:2608.31166},
  year    = {2026}
}

@article{abbadi2026constant,
  author  = {Omar Abbadi and Rida Laraki and Panayotis Mertikopoulos},
  title   = {Constant Regret in General Games via Higher-Order Optimism},
  journal = {arXiv preprint arXiv:2609.04113},
  year    = {2026}
}

@article{soleymani2026multiplicative,
  author  = {Ashkan Soleymani and Georgios Piliouras},
  title   = {Multiplicative Optimism for Constant Regret in Games},
  journal = {arXiv preprint arXiv:2609.21976},
  year    = {2026}
}

@article{ha2026horizon,
  author  = {Junsoo Ha},
  title   = {A Horizon-Independent Regret Bound for {Optimistic Hedge} in General-Sum Games},
  journal = {arXiv preprint arXiv:2609.22839},
  year    = {2026}
}

@article{mai2026constant,
  author  = {Tung Mai},
  title   = {Constant Swap Regret in General-Sum Games via Two-Scale Higher-Order Optimism},
  journal = {arXiv preprint arXiv:2609.16751},
  year    = {2026}
}

\end{document}